\documentclass{article}

\usepackage[preprint]{neurips_2025}

\usepackage[utf8]{inputenc} % allow utf-8 input
\usepackage[T1]{fontenc}    % use 8-bit T1 fonts
\usepackage{hyperref}       % hyperlinks
\usepackage{url}            % simple URL typesetting
\usepackage{booktabs}       % professional-quality tables
\usepackage{booktabs}
\usepackage{graphicx}       % for \resizebox
\usepackage{amsfonts}       % blackboard math symbols
\usepackage{nicefrac}       % compact symbols for 1/2, etc.
\usepackage{microtype}      % microtypography
\usepackage{xcolor}         % colors

\usepackage{amsmath,amssymb,amsthm}

\newtheorem{theorem}{Theorem}

\newtheorem{proposition}[theorem]{Proposition}

\usepackage{algorithm}
\usepackage{algpseudocode}
\usepackage{graphicx}
\usepackage{subcaption}

\defcitealias{starvla2025}{StarVLA, 2026}

\title{VLA-GSE: Boosting Parameter-Efficient Fine-Tuning in VLA with Generalized and Specialized Experts}

\author{
  Yuhua Jiang$^{1,2}$\thanks{Work conducted at Microsoft Research Asia.}, Junjie Lu$^{1}$, Xinyao Qin$^{1,2}$, Xiaoyu Chen$^{1,2}$, Kaixin Wang$^{1}$, \\
  \textbf{Feifei Gao$^{2}$, Li Zhao$^{1}$\thanks{Corresponding author.}} \\
  $^{1}$Microsoft Research Asia \\
  $^{2}$Tsinghua University
}

\begin{document}

\maketitle

\begin{abstract}
Vision-language-action (VLA) models inherit rich visual-semantic priors from pre-trained vision-language backbones, but adapting them to robotic control remains challenging.
Full fine-tuning (FFT) is prone to overfitting on downstream robotic data and catastrophic forgetting of pretrained vision-language capabilities. Parameter-efficient fine-tuning (PEFT) better preserves pre-trained knowledge,
yet existing PEFT methods still struggle to adapt effectively 
to robot control tasks.
To address this gap, 
we propose \textbf{VLA-GSE}, a parameter-efficient VLA fine-tuning framework that improves control adaptation while retaining PEFT's knowledge preservation advantage.
Specifically, VLA-GSE (Generalized and Specialized Experts) is initialized by spectrally decomposing the frozen backbone, assigning leading singular components to generalized experts (shared experts) and disjoint residual components to specialized experts (routed experts). This decomposition improves adaptation capacity under a fixed trainable-parameter budget.
Across comparable-budget comparison, VLA-GSE updates only 2.51\% of full model parameters and consistently outperforms strong FFT and PEFT baselines. It achieves 81.2\% average zero-shot success on LIBERO-Plus, preserves pre-trained VLM capability comparably to LoRA on multimodal understanding benchmarks, and improves real-world manipulation success under multiple distribution shifts.
Code is available at \url{https://github.com/YuhuaJiang2002/VLA-GSE}. 
\end{abstract}

\section{Introduction}
%\li{I rewrite the introduction part. }

Vision-language-action (VLA) models have emerged as a compelling paradigm for general-purpose robotic control, where pre-trained vision-language models (VLMs) are adapted into embodied policies~\citep{firoozi2023foundation, ma2024survey, zitkovich2023rt2, kim2025openvla, zhang2026vlm4vla, wu2026vlanext, chen2025villa}. By inheriting strong multimodal priors from foundation VLMs \citep{dai2023instructblip}, these models map visual observations and language instructions to low-level actions, enabling substantial progress in long-horizon manipulation and boarder generalization \citep{gao2025vlaos, chen2024spatialvlm, li2024vision, mees2022calvin, ahn2023saycan, huang2025thinkact, shridhar2023peract, black2024pi0, asyncVLA, chen2024igor}.

Despite this progress, robust VLM-to-VLA adaptation remains fundamentally challenging. Full fine-tuning (FFT) offers high expressivity but often overfits limited robotic data and degrades pre-trained vision-language competence through catastrophic forgetting \citep{yang2026abotm0, li2024towards, wen2025diffusionvla, hancock2025actions, dey2024revla, yang2025instructvla, yu2026twinbrainvla}. Parameter-efficient fine-tuning (PEFT), such as low-rank adaptation (LoRA), is more stable in preserving pre-trained capability \citep{hancock2025actions}, yet existing PEFT approaches frequently under-adapt when precise control is required. This raises a central question: \textit{can PEFT be made sufficiently expressive for precise robot-control adaptation while retaining its knowledge-preservation advantage?}

To address this question, we propose \textbf{VLA-GSE} (\textbf{G}eneralized and \textbf{S}pecialized \textbf{E}xperts), a parameter-efficient VLA fine-tuning framework designed to strengthen adaptation without sacrificing the knowledge retention behavior of PEFT. VLA-GSE is built from a spectral decomposition of the frozen backbone: leading singular components initialize generalized experts (shared experts), while disjoint residual components initialize specialized experts (routed experts). This decomposition concentrates shared capacity in always-available experts, reduces redundancy in routed adaptation, and enables more flexible expert composition for targeted control behaviors.

This design introduces two optimization issues. First, spectral initialization induces large singular-value imbalance across specialized experts, which leads to uneven update magnitudes and weakly trained experts. We address this with \emph{expert-wise gradient scale balancing}, which compensates for spectral imbalance and harmonizes optimization scales. Second, because routing is input-dependent, the equivalent weight becomes sample-dependent and can drift from the frozen backbone at initialization. We address this with a \emph{backbone weight adjustment} mechanism that preserves the backbone-equivalent weight in expectation.

Empirically, VLA-GSE updates only 2.51\% of full model parameters yet consistently outperforms strong baselines under a comparable trainable-parameter budget. On LIBERO-Plus, it achieves 81.2\% average zero-shot success, exceeding FFT by 6.3\% and the strongest PEFT baseline by 4.4\%, with strong robustness under diverse perturbations.
On multimodal understanding benchmarks, VLA-GSE preserves pre-trained VLM capability comparably to LoRA and outperforms both FFT and co-trained VLA baselines.
In real-world evaluation across four manipulation tasks under four distribution shifts, VLA-GSE reaches 82.5\% success, outperforming FFT by 16.7\%.

\section{Related Work and Background}
\label{sec:related_work}

\paragraph{Finetuning VLMs into VLAs.} 
Recent VLA models adapt pre-trained VLMs into robot policies by retaining the multimodal backbone and training additional action modules on robot trajectories \citep{zhang2026vlm4vla,wu2026vlanext,gao2025vlaos}. Some recent architectures, such as OpenVLA-OFT-style designs, further improve inference efficiency via parallel decoding.
Since FFT methods are memory consuming, 
LoRA is explored in VLA fine-tuning works
\citep{kim2025openvla,hancock2025actions,omaisan2025towardsaccessiblephysicalai}. 
% This has motivated parameter-efficient adaptation strategies that preserve backbone knowledge while updating only a small subset of parameters.
In particular, {OpenVLA}~\citep{kim2025openvla} is among the earliest VLA works to adopt LoRA in VLA fine-tuning for lower training memory consumption.
VLM2VLA \citep{hancock2025actions} provides evidence that LoRA better preserves pre-trained VLM knowledge during VLA adaptation. 
% Existing PEFT-based VLA finetuning methods are still largely limited to LoRA-style adaptation \citep{hancock2025actions,omaisan2025towardsaccessiblephysicalai}.
Although LoRA-based methods better preserve pre-trained vision-language knowledge than FFT, they often lack sufficient adaptation capacity for precise embodied control. Our method explicitly targets both goals: preserving pretrained multimodal knowledge while enabling stronger downstream control adaptation.
% finetuning vlm into vla （OpenVLA-OFT 类似的架构）提一下 parallel decoding
% only peft in VLA finetuning is LoRA in \citep{hancock2025actions，omaisan2025towardsaccessiblephysicalai} 重点讲讲和这篇区别, 他们可以preserve knowledge 但是control 的性能不太行

\paragraph{Parameter-Efficient Fine-Tuning.}
PEFT has become a practical alternative to FFT mainly in LLM finetuning, with LoRA as the canonical example \citep{hu2022lora}. Recent work shows that adapter structure and initialization critically affect adaptation quality: SVD-based methods such as PiSSA, MiLoRA, KaSA, and GOAT exploit different spectral components to better preserve pretrained knowledge or expose task-relevant subspaces \citep{meng2024pissa,wang2024milora,wang2024kasa,fan2025goat}, while MoE-style methods such as MoLoRA, AdaMoLE, and HydraLoRA improve capacity through routed low-rank experts \citep{zadouri2024pushing,liu2024adamole,tian2024hydralora}.
However, existing PEFT-based VLA finetuning methods are still largely limited to LoRA-style adaptation \citep{hancock2025actions,omaisan2025towardsaccessiblephysicalai}, 
and 
PEFT remains under-explored in VLA finetuning, where VLA models require stronger adaptation for precise embodied control. 
% \li{remove knowledge preservation, what LLM needs most is compute efficiency. Just address VLA needs stronger adaptation for precise embodied control.} 
Motivated by this difference, VLA-GSE is tailored to the control setting, combining a shared generalized expert with routed specialized experts to support both transferable adaptation and input-specific control refinement. We further introduce gradient scale balancing to stabilize optimization across experts initialized from different spectral components.
% peft 主要是在LLM上面做的，PEFT 在VLA领域 under explored，VLA-GSE不一样也是因为task不一样，对control adaptation的要求更强，改进是为了VLA 精细控制设计的

\begin{figure*}[t]
\centering
\includegraphics[width=0.8\textwidth]{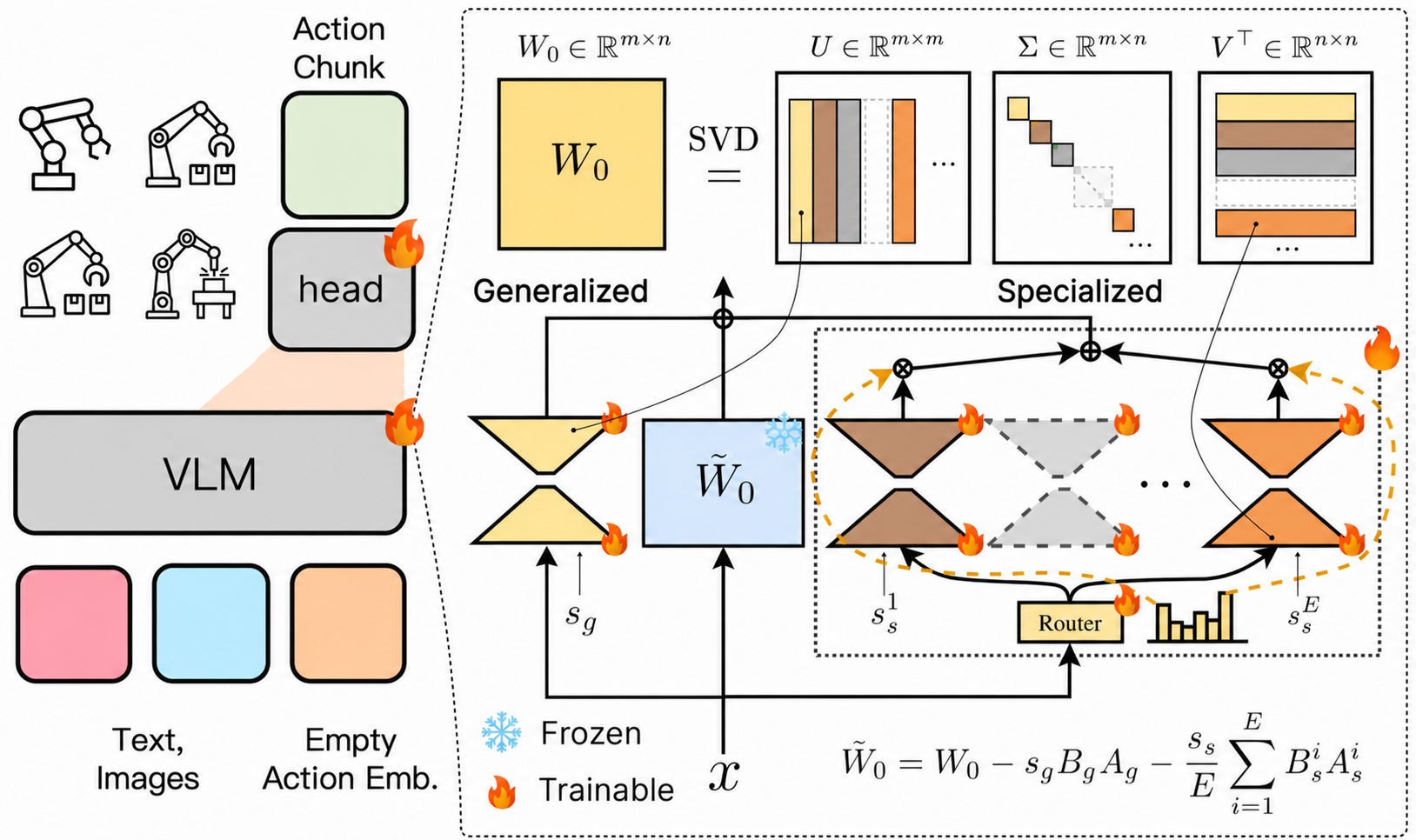}
\caption{\textbf{Overview of the VLA-GSE framework.} (Top Right) VLA-GSE uses an SVD-based adaptive priors scheme to initialize generalized experts and multiple specialized experts based on sorted singular value segments. (Bottom Right) During the forward pass, the output is formed by summing the input transformations induced by the adjusted frozen pre-trained weight $\tilde{W}_0$, the generalized expert, and the dynamically selected top-$k$ specialized experts. Unlike prior MoE LoRA, the generalized expert is always active and initialized from the SVD-based priors. GSE is only applied for the pretrained weight $W_0$ in the VLM. 
}
\label{gse}
\end{figure*}

\section{Method}
\label{sec:method}

% \li{consider write something here,
% to illustrate your method, highlight the key design against other PEFT method, and highlight the technique challenges...}

\paragraph{Overview.}
% \jiang{
Our model follows a standard VLM-to-VLA pipeline, where a pre-trained VLM encodes image observations and language instructions, and an OpenVLA-OFT-style action head is used to generate robot actions \citep{kim2025fine}.
The action head is fully fine-tuned for control and we apply PEFT only to the VLM backbone: for a pre-trained weight $W_0 \in \mathbb{R}^{m \times n}$, we learn a structured low-rank update while keeping most backbone parameters frozen. On top of this foundation, we propose VLA-GSE, which decomposes backbone adaptation into an always-on generalized expert and multiple routed specialized experts. The design is tailored to VLA finetuning, where the model must both preserve pretrained vision-language knowledge and support precise robot control adaptation. VLA-GSE addresses this through generalized-specialized decomposition, expert-wise gradient scale balancing, and backbone weight adjustment.
% }
% Ours 基于StarVLA 的架构，选择了Openvla-OFT 的action head，PEFT 的VLM + FFT 的 action head MLP生成action， 
% 和 peft 的 background （只涉及 PEFT for weight in VLM，FFT for action head， LoRA 以及 variants）
% peft 针对的是 $W_0 \in \mathbb{R}^{m \times n}$ denote a weight in VLM 
% 以及GSE导致的chanllenge和解决方法

% \label{subsec:motivation}

% Downstream adaptation of VLA models requires balancing two competing objectives: preserving broad, domain-general knowledge from pre-training and acquiring fine-grained manipulation skills for specific robotic tasks. Existing SVD-based MoE methods \citep{fan2025goat} typically partition singular components uniformly across experts, but such a homogeneous design provides no explicit mechanism to ensure consistent use of foundational knowledge across tokens and modalities.

% Motivated by a spectral view of pre-trained weights, we propose VLA-GSE, which decouples adaptation into an always-active \textbf{Generalized Expert}, initialized from the leading singular components to preserve domain-general backbone knowledge, and a set of sparsely routed \textbf{Specialized Experts}, initialized from disjoint residual components to capture context-dependent adaptation. This design improves parameter-efficient transfer while mitigating the loss of general knowledge, as illustrated in Figure~\ref{gse}.

\subsection{Generalized and Specialized Expert Initialization via SVD}
We initialize VLA-GSE from the SVD of the frozen pre-trained weight matrix $W_0 \in \mathbb{R}^{m \times n}$: $W_0 = U \Sigma V^\top$.
Sorting singular values in descending order induces a spectral decomposition of $W_0$, from which we allocate the leading components to the \emph{generalized expert} and the \emph{specialized experts}, as illustrated in Figure~\ref{gse}.

\paragraph{Generalized Expert.} 
We allocate the largest $r_g$ singular values to the generalized expert. This expert is tasked with capturing common knowledge and mitigating redundancy in routed experts.
% preserving and updating the most critical and generalized pre-trained knowledge in VLMs.\li{revise this sentence. copy motivation for alway-on expert from deepseekMoE.} 
Its initialization is formulated as follows:
\begin{equation}
    B_g = \sqrt{\frac{1}{s_g}} U_{1:r_g} \Sigma_{1:r_g}^{1/2} \in \mathbb{R}^{m \times r_g}, \quad A_g = \sqrt{\frac{1}{s_g}} \Sigma_{1:r_g}^{1/2} (V_{1:r_g})^\top \in \mathbb{R}^{r_g \times n}, \quad W_g = s_g B_g A_g ,
\end{equation}
where $U_{1:r_g}$, $\Sigma_{1:r_g}$, and $V_{1:r_g}$ correspond to the top segment of the SVD, $s_g$ is the scaling factor for the generalized expert, and $W_g$ is the generalized expert's effective weight. 
This design is motivated by the observation that leading singular components capture the dominant energy of the pre-trained weight matrix,
which makes the always-on adaptation more effective \citep{meng2024pissa}. 
% which empirically corresponds to domain-general features in pretrained VLMs.
% \li{revise this sentence accordingly}

\paragraph{Specialized Experts.} 
The remaining dominant singular components are evenly partitioned to initialize $E$ specialized experts. Under the adaptive prior strategy, we consider the SVD components indexed by $[r_g+1,\, r_g+dE]$ and assign each expert a contiguous block of rank $d$. Formally, let
$\mathcal{I}_i \triangleq \{\, r_g + (i-1)d + 1,\; \ldots,\; r_g + id \,\}, i \in \{1,2,\dots,E\}$.
The $i$-th specialized expert is then initialized by
$U_i = U_{:,\,\mathcal{I}_i}, 
% \qquad
\Sigma_i = \Sigma_{\mathcal{I}_i,\,\mathcal{I}_i},
% \qquad
V_i^\top = V^\top_{\mathcal{I}_i,\,:}$.
Using these designated segments, the initialization is formulated as:
\begin{equation}
    B_s^i = \sqrt{\frac{1}{s_s^i}} U_i \Sigma_i^{1/2} \in \mathbb{R}^{m \times d}, \quad A_s^i = \sqrt{\frac{1}{s_s^i}} \Sigma_i^{1/2} V_i^\top \in \mathbb{R}^{d \times n}, \quad W_s^i = s_s^i B_s^i A_s^i,
\end{equation}
where $s_s^i$ is the $i$-th expert's scaling factor that normalizes each expert’s scale based on the spectral magnitude of its assigned segment, and $W_s^i$ is the effective weight of the $i$-th specialized expert.
% \citep{fan2025goat}

\paragraph{Specialized Expert Selection Method}
During the forward pass, a router $W_z \in \mathbb{R}^{E \times n}$ maps the input $x$ to routing logits $z(x)$, which determine a sparse subset of specialized experts for $x$: 
    $p^i(x) = \frac{\exp(z^i(x))}{\sum_{j=1}^{E} \exp(z^j(x))}, z(x) = W_z x $.
Let $\Omega_k(x)$ denote the indices of the top-$k$ $p^i(x)$. The routing weights $w^i(x)$ are explicitly defined by applying a softmax function exclusively over the selected top-$k$ logits:
\begin{equation}
    w^i(x) = \begin{cases} 
    \frac{\exp(z^i(x))}{\sum_{j \in \Omega_k(x)} \exp(z^j(x))}, & \text{if } i \in \Omega_k(x), \\ 
    0, & \text{otherwise.} 
    \end{cases}
\end{equation}

\paragraph{Load Balancing Auxiliary Loss}
\label{sec:load_balancing}

To prevent the sparse router over-selecting a small subset of specialized experts, we introduce an auxiliary load-balancing loss to encourage balanced expert utilization across GSE blocks. % \citep{dai2024deepseekmoe, fedus2022switch}
For the $l$-th GSE block, let $f_i^{(l)}$ denote the fraction of tokens assigned to expert $i$, and let $P_i^{(l)}$ denote its average routing probability over the batch. We define
\begin{equation}
\begin{aligned}
\mathcal{L}_{\text{bal}}^{(l)} &= E \sum_{i=1}^{E} f_i^{(l)} P_i^{(l)}, \qquad
\mathcal{L}_{\text{final}} = \| \hat{{a}} - {a} \|_1 + \alpha \sum_{l=1}^{L} \mathcal{L}_{\text{bal}}^{(l)}.
\end{aligned}
\end{equation}
The first term in $\mathcal{L}_{\text{final}}$
is the action regression loss and the second term regularizes expert usage. 
Detailed explanation of load balancing auxiliary loss
is shown in Appendix~\ref{app:load_balancing}

% \subsection{Optimization Balancing}
\subsection{Gradient Scale Balancing.}
To avoid optimization imbalance across specialized experts, we analyze their gradient scales at initialization.
% Let $g \in \mathbb{R}^{m \times n}$ denote the gradient of the loss with respect to the overall aggregated weight, and 
Let $g_A^i$ and $g_B^i$ denote the localized gradients of the $i$-th specialized expert with respect to $A_s^i$ and $B_s^i$, respectively.
To characterize the expected gradient scale, we introduce the left and right second-moment matrices of the gradient:
$G_L \triangleq \mathbb{E}[g g^\top],
G_R \triangleq \mathbb{E}[g^\top g]$.
Here, $G_L$ and $G_R$ describe the second-moment structure of $g$ in the left and right subspaces, respectively. 
The following theorem gives the exact initialization-time gradient scale for each expert.

\begin{theorem}[Expert-wise gradient scale balancing]
\label{thm:scaling_factor}
Let $w^i(x)$ denote the input-dependent routing coefficient multiplying the
contribution of the $i$-th specialized expert in the aggregated equivalent
weight for input $x$.
Under the idealized balanced-routing condition at initialization, define
$\rho_w := \mathbb{E}_x\!\left[(w^i(x))^2\right],$
and assume that $\rho_w$ is identical across all specialized experts
$i \in \{1,\dots,E\}$.
Then, for the $i$-th specialized expert, the expected squared Frobenius norms
of its localized gradients satisfy
\begin{equation}
\mathbb{E}\!\left[\|g_A^i\|_F^2\right]
=
\rho_w \, s_s^i \, \alpha_i^L,
\qquad
\mathbb{E}\!\left[\|g_B^i\|_F^2\right]
=
\rho_w \, s_s^i \, \alpha_i^R ,
\end{equation}
where
\begin{equation}
\alpha_i^L = \mathrm{Tr}\!\left(\Sigma_i U_i^\top G_L U_i\right),
\qquad
\alpha_i^R = \mathrm{Tr}\!\left(\Sigma_i V_i^\top G_R V_i\right).
\end{equation}
Therefore, under balanced routing, the optimization scale of expert $i$ is
determined jointly by its spectral magnitude $\Sigma_i$, the projected
second-order gradient energy on its corresponding singular subspaces, and a
shared routing second-moment factor $\rho_w$.
Moreover, if there exist constants $\kappa_L,\kappa_R > 0$
such that\footnote{This assumption is automatically satisfied under
$G_L =\mathbb{E}[g g^\top] \propto I$ and $G_R =\mathbb{E}[g^\top g] \propto I$, which is
the commonly used isotropic gradient variance assumption
\citep[used in Theorem 1 of][]{jastrzebski2018three}.}
\begin{equation}
\alpha_i^L = \kappa_L \, \mathrm{Tr}(\Sigma_i),
\qquad
\alpha_i^R = \kappa_R \, \mathrm{Tr}(\Sigma_i),
\qquad
\forall i \in \{1,\dots,E\},
\label{eq:projected_uniformity_main}
\end{equation}
then choosing
\begin{equation}
s_s^i
=
s_{\mathrm{base}} \cdot \frac{C}{\mathrm{Tr}(\Sigma_i)}
\label{eq:trace_inverse_scaling_main}
\end{equation}
equalizes both $\mathbb{E}[\|g_A^i\|_F^2]$ and
$\mathbb{E}[\|g_B^i\|_F^2]$ across all specialized experts, since the
additional factor $\rho_w$ is shared by all experts under the balanced-routing
assumption. Here $s_{\mathrm{base}}$ is a global base scaling hyperparameter
and $C$ is a normalization constant, e.g.,
$C = \frac{1}{E}\sum_{j=1}^E \mathrm{Tr}(\Sigma_j)$.
\end{theorem}

Theorem~\ref{thm:scaling_factor} shows that the trace-inverse scaling rule in Eq.~\eqref{eq:trace_inverse_scaling_main} is justified whenever the projected second-order gradient statistics are approximately uniform across experts, as specified in Eq.~\eqref{eq:projected_uniformity_main}. 
% Load balancing 在文
The detailed proof is provided in Appendix~\ref{sec:proof_scaling}.
% Notably, Eq.~\eqref{eq:projected_uniformity_main} is automatically satisfied when $\mathbb{E}[g g^\top] \propto I$ and $\mathbb{E}[g^\top g] \propto I$, which is the commonly used isotropic gradient variance assumption in gradient descent analysis \citep[see Theorem 1 of][]{jastrzebski2018three}.
% In practice, Eq.~\eqref{eq:trace_inverse_scaling_main} prevents experts associated with smaller singular-value blocks from receiving disproportionately weak updates at initialization, and thus stabilizes training.

\subsection{Backbone Weight Adjustment}

Unlike LoRA-style zero perturbation at initialization, VLA-GSE induces a non-zero low-rank perturbation at initialization, which would otherwise shift the weight away from $W_0$ before fine-tuning. 
To preserve the pre-trained backbone as much as possible at initialization, we compensate for the non-zero perturbation introduced by the generalized and specialized experts by adjusting the frozen backbone weight from $W_0$ to $\tilde{W}_0$.
Let $W_{eq}(x)$ denote the input-dependent equivalent weight of the GSE block. Then, it follows that
\begin{equation}
y = W_{eq}(x)x,\qquad
W_{eq}(x)=\tilde{W}_0+s_g B_gA_g+\sum_{i=1}^{E} w^i(x)\, s_s^i B_s^i A_s^i .
\end{equation}

% We therefore choose $\tilde{W}_0$ such that the initialized effective weight remains aligned with the pre-trained backbone.

In principle, one may desire an exact sample-wise equality:
    $W_{eq}(x)=W_0, \forall x$.
However, this is generally impossible because the selection of specialized experts varies with $x$ through $w^i(x)$. 
Accordingly, we aim to maintain consistency with the pretrained weight \emph{in expectation} under the routing distribution at initialization.
At random initialization, the router is symmetric across specialized experts in expectation.
Since they also satisfy $\sum_{i=1}^{E} w^i(x)=1$, we have
    $\mathbb{E}_x[w^i(x)] = \frac{1}{E},  \forall i \in \{1,\dots,E\}$.
It follows that the expected weight of the specialized experts is
\begin{equation}
    \mathbb{E}_x \left[ \sum_{i=1}^{E} w^i(x)\, s_s^i B_s^i A_s^i \right]
    =
    \sum_{i=1}^{E} \mathbb{E}_x[w^i(x)]\, s_s^i B_s^i A_s^i
    =
    \frac{1}{E} \sum_{i=1}^{E} s_s^i B_s^i A_s^i.
\end{equation}

This motivates defining the expected residual shift introduced at initialization:
\begin{equation}
    W_{res}
    =
    s_g B_g A_g
    +
    \frac{1}{E} \sum_{i=1}^{E} s_s^i B_s^i A_s^i,
\end{equation}
where $W_{res}$ denotes the total expected offset of the initialized GSE weight relative to the original backbone. We then compensate for this offset by redefining the adjusted frozen weight as
\begin{equation}
    \tilde{W}_0 = W_0 - W_{res} = W_0 - s_g B_g A_g
    -
    \frac{1}{E} \sum_{i=1}^{E} s_s^i B_s^i A_s^i.
    \label{adjustment}
\end{equation}
Under the backbone weight adjustment in Eq.~\ref{adjustment}, 
$W_{eq}(x)$ is aligned with $W_0$ in expectation at initialization, where the detailed proof is shown in Appendix~\ref{proof2}. 
The full initialization and forward-pass procedure is provided in Appendix~\ref{sec:appendix_algorithm}. 

\section{Simulation Experiments}

\subsection{Settings}

\textbf{Model Architecture and Initialization.} 
In our experiments, we adopt Qwen3-VL-4B-Instruct as the foundation VLM backbone. To efficiently adapt the VLM into a VLA model, we apply our proposed VLA-GSE framework. Specifically, the low-rank dimension is set to $r=16$. Each GSE block consists of 8 rank-2 experts in total, which explicitly includes 1 generalized expert to anchor the foundational representations, and 7 specialized experts utilizing a Top-$2$ gating strategy. Scaling across the architecture, this configuration deploys a total of 356 generalized experts and 2492 specialized experts. The initialization strictly follows the SVD-based approach. For the action head, we use the MLP architecture and parallel decoding method from OpenVLA-OFT \citep{kim2025fine}.

% \textbf{Training Dataset and Environment.} 
% We evaluate our method on the LIBERO-Plus simulation benchmark and real world environment. The model is trained using the combined training sets from all 4 suites of LIBERO. 

\textbf{Optimization and Hyperparameters.} 
The model is trained from Qwen3-VL-4B-Instruct for a total of 80,000 optimization steps. The batch size per GPU is set to 16, resulting in an effective total batch size of 128 across 8 NVIDIA A100 GPUs. To ensure stable adaptation while preventing the catastrophic forgetting of pre-trained visual-semantic knowledge, we apply decoupled learning rates: the learning rate for the GSE parameters in VLM is set to $1 \times 10^{-5}$, while the newly initialized action MLP head uses a learning rate of $1 \times 10^{-4}$. Furthermore, the auxiliary load-balancing loss weight $\alpha$ is set to $0.01$.
We set the generalized scaling factor as 
$s_g = 2$ and set the $i$-th specialized scaling factor as 
$s_s^i = \frac{2 \sum_{j=1}^E \mathrm{Tr}(\Sigma_j)}{E \cdot \mathrm{Tr}(\Sigma_i)}$. 
In this way, the scaling of each specialized expert is inversely proportional to its spectral magnitude.
A comprehensive summary of the detailed hyperparameter settings and model configurations is provided in Appendix Table \ref{tab:hyperparameters}.

\textbf{Parameter Efficiency.} 
Training VLA-GSE is highly parameter-efficient. Among the 4,551.85M total model parameters, we freeze 4,437.82M non-GSE parameters from the base VLM and optimize only 114.04M parameters (2.51\% of the full model). Specifically, the trainable parameters consist of 48.41M parameters in the GSE modules and 65.62M parameters in the action head.

\subsection{LIBERO-Plus Benchmark Results}

We use LIBERO-Plus \citep{libero_plus} as the simulation benchmark, 
which evaluates VLA models' generalization ability across seven perturbation dimensions.
All models are trained using the combined training sets from all 4 suites of LIBERO \citep{libero}.
Table~\ref{1} shows that VLA-GSE achieves the best overall zero-shot performance, obtaining 81.2\% average success and outperforming the strongest prior baselines, including ABot-M0 (80.5\%) and VLANeXt (80.1\%). The improvement is broad rather than isolated: VLA-GSE delivers the strongest average robustness on Robot, Language, Light, and Background, while matching the best Layout performance. At the suite level, it attains 90.3\% on Spatial and 86.2\% on Object, and remains stable on the more challenging Goal and Long suites with 74.2\% and 74.1\%, respectively. Overall, these results indicate that VLA-GSE improves not only aggregate success rate, but also robustness under diverse test-time perturbations. The full suite-wise results of {VLA-GSE} are provided in Appendix~\ref{suitewise}.
\begin{table*}[t]
  \centering
  \caption{LIBERO-Plus benchmark zero-shot performance. Results are shown in success rate (\%).}
  \label{1}
  \resizebox{\textwidth}{!}{%
    \begin{tabular}{l|ccccccc|c}
      \toprule
      Model & Camera & Robot & Language & Light & Background & Noise & Layout & Total \\
      \midrule
      OpenVLA \citep{kim2025openvla}   & 0.8  & 3.5  & 23.0 & 8.1  & 34.8 & 15.2 & 28.5 & 15.6 \\
      WorldVLA \citep{cen2025worldvla}  & 0.1  & 27.9 & 41.6 & 43.7 & 17.1 & 10.9 & 38.0 & 25.0 \\
      NORA \citep{hung2025nora}      & 2.2  & 37.0 & 65.1 & 45.7 & 58.6 & 12.8 & 62.1 & 39.0 \\
      UniVLA \citep{bu2025univla}    & 1.8  & 46.2 & 69.6 & 69.0 & 81.0 & 21.2 & 31.9 & 42.9 \\
      $\pi_0$ \citep{black2024pi0}   & 13.8 & 6.0  & 58.8 & 85.0 & 81.4 & 79.0 & 68.9 & 53.6 \\
      $\pi_0$-Fast \citep{pertsch2025fast} & 65.1 & 21.6 & 61.0 & 73.2 & 73.2 & 74.4 & 68.8 & 61.6 \\
      OpenVLA-OFT \citep{kim2025fine} & 56.4 & 31.9 & 79.5 & 88.7 & 93.3 & 75.8 & 74.2 & 69.6 \\
      RIPT-VLA \citep{tan2025interactive} & 55.2 & 31.2 & 77.6 & 88.4 & 91.6 & 73.5 & 74.2 & 68.4 \\
      ABot-M0 \citep{yang2026abotm0} & 60.4 & 67.9 & 86.4 & 96.2 & 91.6 & 86.4 & \textbf{82.6} & 80.5 \\
      VLANeXt \citep{wu2026vlanext}  & \textbf{71.6} & 63.4 & 81.4 & 93.7 & 91.8 & \textbf{89.5} & 77.7 & 80.1 \\
      \textbf{VLA-GSE (Ours)} & 64.4 & \textbf{68.5} & \textbf{88.8} & \textbf{97.3} & \textbf{97.3} & 79.4 & \textbf{82.6} & \textbf{81.2} \\
      \bottomrule
    \end{tabular}
  }
\end{table*}

Table~\ref{tab:libero_plus_ft} compares fine-tuning strategies in LIBERO-Plus zero-shot generalization using the Qwen3-VL-4B-Instruct backbone. For a fair comparison, all PEFT methods use roughly 2.51\% (114M) trainable parameters and are independently tuned for optimal performance.
% \li{The parameter budgets are not strictly matched, so we may need to adjust the wording in the abstract and introduction.}
Notably, FFT (74.9\%) underperforms several PEFT baselines due to \textit{catastrophic forgetting} \citep{hancock2025actions,yu2026twinbrainvla}, where updating all parameters overwrites the backbone's generalized priors and degrades zero-shot capabilities. Our method, VLA-GSE, successfully mitigates this by balancing knowledge retention with efficient adaptation, achieving the highest overall success rate of {81.2\%}. It outperforms FFT by {6.3} points and the strongest baseline, GOAT, by {4.4} points on average in all seven zero-shot dimensions. These results demonstrate that VLA-GSE is significantly more robust and effective for VLA fine-tuning than conventional LoRA-style, SVD-based, or LoRA-MoE variants.

\begin{table*}[t]
  \centering
  \caption{Finetuning method comparison based on LIBERO-Plus benchmark zero-shot performance. Results are shown in success rate (\%).}
  \label{tab:libero_plus_ft}
  \resizebox{\textwidth}{!}{%
    \begin{tabular}{l|c|ccccccc|c}
      \toprule
      Model & Params (\%) & Camera & Robot & Language & Light & Background & Noise & Layout & Total \\
      \midrule
      FFT (Full Fine-Tuning) & 100.00 & 45.8 & 60.5 & 86.9 & 95.9 & 94.7 & 74.1 & 79.2 & 74.9 \\
      \midrule
      LoRA~\citep{hu2022lora} & 2.58 & 44.3 & 53.6 & 82.1 & 92.9 & 91.1 & 67.1 & 67.2 & 69.2 \\
      rsLoRA~\citep{kalajdzievski2023rslora} & 2.52 & 46.0 & 58.9 & 85.1 & 93.6 & 92.4 & 71.6 & 76.5 & 73.1 \\
      DoRA~\citep{liu2024dora} & 2.54 & 49.7 & 61.0 & 86.6 & 95.7 & 94.4 & 73.4 & 78.6 & 75.3 \\
      PiSSA~\citep{meng2024pissa} & 2.54 & 48.3 & 59.4 & 86.0 & 95.4 & 94.0 & 72.8 & 78.1 & 74.5 \\
      MoLoRA~\citep{zadouri2024pushing} & 2.56 & 52.2 & 61.5 & 87.5 & 96.0 & 95.0 & 74.2 & 79.0 & 76.2 \\
      AdaMoLE~\citep{liu2024adamole} & 2.53 & 50.2 & 60.8 & 87.8 & 95.8 & 94.4 & 73.1 & 78.6 & 75.5 \\
      HydraLoRA~\citep{tian2024hydralora} & 2.55 & 51.3 & 60.1 & 86.4 & 95.4 & 94.1 & 73.0 & 78.2 & 75.2 \\
      MiLoRA~\citep{wang2025milora} & 2.52 & 51.1 & 60.1 & 86.3 & 94.9 & 93.9 & 73.0 & 77.7 & 75.0 \\
      GOAT~\citep{fan2025goat} & 2.54 & 54.5 & 61.8 & 87.7 & 96.2 & 95.2 & 74.6 & 79.3 & 76.8 \\
      \textbf{VLA-GSE (Ours)} & 2.51 & \textbf{64.4} & \textbf{68.5} & \textbf{88.8} & \textbf{97.3} & \textbf{97.3} & \textbf{79.4} & \textbf{82.6} & \textbf{81.2} \\
      \bottomrule
    \end{tabular}
  }
\end{table*}

\begin{figure}[t]
    \centering
    \includegraphics[width=\textwidth]{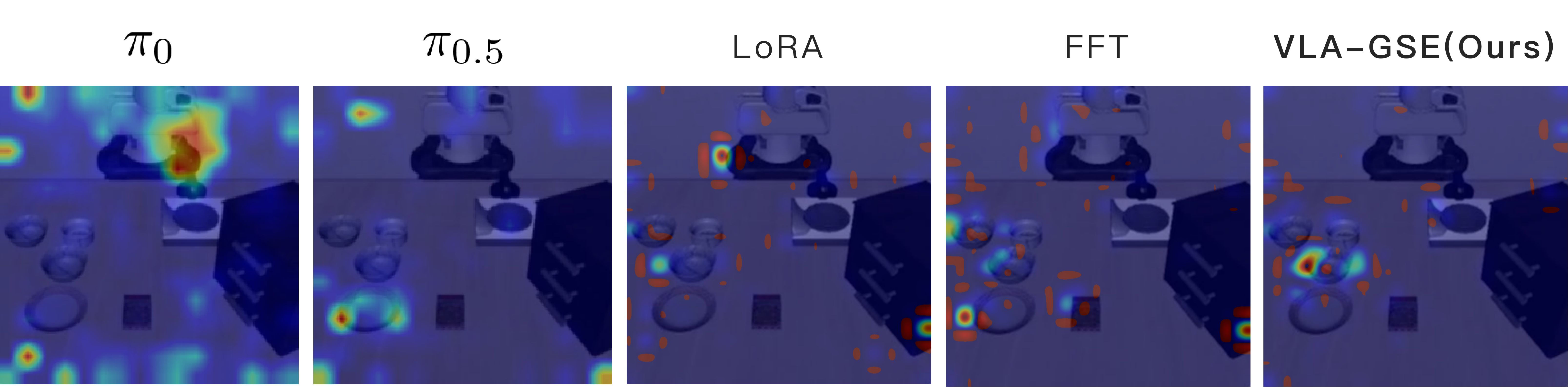} 
    \caption{Attention map visualizations for different models under the LIBERO-Plus \textbf{dark lighting} environment. The instruction is ``Pick up the black bowl between the plate and the ramekin and place it on the plate.'' Our VLA-GSE method exhibits precise attention on the \textbf{graspable edges} of the target object, whereas $\pi_{0}$, $\pi_{0.5}$, FFT, and LoRA show scattered attention or focus on incorrect objects.}
    \label{fig:attention_map}
\end{figure}

% and action execution\li{Is this wording accurate?}
% and final grasp outcomes
We qualitatively illustrate the attention maps in the challenging \textbf{dark lighting} setting of LIBERO-Plus. Figure~\ref{fig:attention_map} shows the last transformer layer attention maps for different models under the instruction: ``Pick up the black bowl between the plate and the ramekin and place it on the plate.'' Both FFT and LoRA are fine-tuned from the same Qwen3-VL-4B-Instruct backbone. VLA-GSE exhibits concentrated attention on the \textbf{graspable edges} of the correct target object and successfully completes the task, indicating that it not only identifies the target specified by the dense spatial instruction, but also captures the manipulation affordance required for grasping, namely the appropriate edge positions for the gripper. In contrast, the baselines show either scattered or mislocalized attention, leading to incorrect object selection or failed grasping. These results suggest that GSE better preserves the backbone's visual-semantic grounding while acquiring task-specific manipulation behavior. More qualitative examples are provided in Figure~\ref{fig:attention_map_appendix}.

\subsection{Knowledge Retention after VLA Fine-Tuning in Terms of Multimodal Understanding}
\label{app:qwen3vl_retention}

We further evaluate the base VLM and VLA-adapted variants on a diverse set of multimodal understanding benchmarks. As shown in Table~\ref{tab:mm_understanding_qwen}, FFT leads to a clear and consistent degradation across benchmarks, indicating severe catastrophic forgetting of pretrained vision--language capability during VLM-to-VLA transfer. A similar trend is also observed for co-trained VLA models such as $\pi_{0.5}$, whose multimodal understanding performance remains far from that of strong general-purpose VLMs, suggesting that preserving VLM knowledge is still challenging under FFT embodied training.

In contrast, both Qwen3-VL-LoRA and {VLA-GSE} remain consistently close to the base Qwen3-VL-4B-Instruct on general multimodal benchmarks, showing substantially better preservation of pretrained multimodal knowledge than FFT. Importantly, VLA-GSE achieves a retention level comparable to standard LoRA-based PEFT, which suggests that its stronger embodied performance does not come from sacrificing the original vision--language capability. Instead, VLA-GSE preserves pretrained VLM knowledge similarly well to LoRA, while providing stronger adaptation capacity for downstream robotic control as shown in Table~\ref{tab:libero_plus_ft}.

% We compare against Prismatic VLM, OpenVLA, ECoT, Gemma-3, MolmoAct, $\pi_{0.5}$, Qwen3-VL-4B-Instruct, Qwen3-VL-FFT and Qwen3-VL-LoRA.
\begin{table*}[t]
\centering
\caption{Multimodal understanding evaluation. Comparison of VLMs and VLAs across multimodal understanding benchmarks. }
\label{tab:mm_understanding_qwen}
\resizebox{\textwidth}{!}{
\begin{tabular}{l|c|cccccccc}
\toprule
Method & \#Params & MMMU & MMStar & OCRBench & MMB-en & DocVQA & InfoVQA & AI2D & RealWorldQA \\
\midrule
\multicolumn{10}{c}{Prismatic VLM Family} \\
\midrule
Prismatic VLM~\citep{karamcheti2024prismatic}      & 7B  & 35.0 & 38.8 & 32.0 & 66.2 & 17.5 & 19.7 & 54.6 & 30.8 \\
OpenVLA~\citep{kim2025openvla}                     & 7B  & 26.3 & 0    & 0    & 0    & 0    & 0    & 0    & 0 \\
ECoT~\citep{ecot2025}                               & 7B  & 26.6 & 0    & 0.01 & 3.7  & 0    & 0    & 0    & 25.6 \\
\midrule
\multicolumn{10}{c}{Gemma-3 Family} \\
\midrule
Gemma-3-4B-IT~\citep{gemma2025gemma3}              & 4B  & 39.3 & 37.1 & 70.2 & 68.6 & 68.8 & 40.9 & 70.5 & 44.0 \\
Gemma-3-12B-IT~\citep{gemma2025gemma3}             & 12B & 46.0 & 46.3 & 75.0 & 76.9 & 80.6 & 50.4 & 78.5 & 50.6 \\
VLM2VLA-AT~\citep{hancock2025actions}                     & 12B & 45.9 & 45.2 & 65.5 & 70.9 & 74.6 & 44.8 & 74.1 & 44.5 \\
VLM2VLA~\citep{hancock2025actions}                        & 12B & 42.7 & 48.0 & 63.9 & 68.5 & 78.4 & 46.2 & 74.0 & 43.3 \\
\midrule
\multicolumn{10}{c}{Other Baseline VLAs} \\
\midrule
MolmoAct~\citep{molmoact2025}                       & 7B  & 28.4 & 1.2  & 52.7 & 55.1 & 58.7 & 41.9 & 2.0  & 8.6 \\
$\pi_{0.5}$~\citep{physicalintelligence2025pi05}    & 3B  & 24.0 & 21.7 & 6.8  & 6.8  & 4.6  & 7.7  & 27.0 & 2.7 \\
\midrule
\multicolumn{10}{c}{Qwen3-VL Family} \\
\midrule
Qwen3-VL-4B-Instruct~\citep{qwen3vl2025}            & 4B & \textbf{53.2} & \textbf{69.8} & \textbf{88.1} & \textbf{85.1} & \textbf{95.3} & \textbf{80.3} & \textbf{84.1} & \textbf{70.9} \\
Qwen3-VL-FFT~\citep{qwen3vl2025}                    & 4B & 35.6 & 31.4 & 43.2 & 50.8 & 55.1 & 37.6 & 48.3 & 33.7 \\
Qwen3-VL-LoRA~\citep{qwen3vl2025}        & 4B & 51.8 & 68.1 & 83.3 & 83.0 & 94.1 & 79.2 & 83.0 & 69.8 \\
\textbf{VLA-GSE (Ours)}                             & 4B & 51.1 & 67.8 & 85.6 & 82.4 & 92.8 & 78.5 & 83.4 & 70.1 \\
\bottomrule
\end{tabular}}
\end{table*}

\subsection{Ablation Study}
To isolate the contributions of our core designs, we conduct an ablation study on the LIBERO-Plus zero-shot Long suite. 
The ``w.o. SVD Initialization'' variant instead uses a standard LoRA-style initialization, with each $B_s^i$ Gaussian initialized and each $A_s^i$ zero initialized.
For the expert routing ablations, we maintain a constant total number of experts to ensure a fair comparison of model capacity.
Specifically, the ``w.o. Generalized Experts'' setting replaces the generalized expert with an additional specialized expert, forcing the model to rely entirely on routing mechanisms. Conversely, the ``w.o. Specialized Experts'' variant replaces all specialized experts with generalized ones. The ``w.o. Auxiliary Loss'' variant removes the load-balancing objective during training. Finally, the ``w.o. Gradient Scaling'' variant replaces the expert-wise trace-inverse scaling in Eq.~\ref{eq:trace_inverse_scaling_main} with a constant scaling factor, setting all $s_s^i$ to 2.

Our full GSE framework achieves the highest overall success rate of 74.1\%, significantly outperforming all ablated variants. The most severe degradation occurs without SVD initialization. Replacing either expert type also causes notable drops, showing that merely increasing parameter capacity is insufficient: the generalized expert provides robust visual-semantic grounding and mitigates overfitting, while the specialized experts decouple complex manipulation strategies under high-variance conditions such as Noise and Layout. Removing the auxiliary loss reduces the overall success rate to 72.0\%, demonstrating that unconstrained routing can lead to expert collapse and suboptimal module utilization. The ``w.o. Gradient Scaling'' variant yields a comparable result of 72.1\%, suggesting that the proposed scaling initialization in Theorem~\ref{thm:scaling_factor} is also important for balancing expert-wise optimization dynamics; without it, uneven gradient magnitudes can cause certain experts to dominate training while others remain under-optimized.

\begin{table*}[t]
  \centering
  \caption{Ablation study on Long suite of LIBERO-Plus benchmark zero-shot performance. Results are shown in success rate (\%).}
  \resizebox{\textwidth}{!}{%
    \begin{tabular}{l|ccccccc|c}
      \toprule 
      Model & Camera & Robot & Language & Light & Background & Noise & Layout & Total \\
      \midrule
      w.o. SVD Initialization  & 25.8 & 41.9 & 80.8 & 84.8 & 88.1 & 51.0 & 70.3 & 60.9 \\
      w.o. Generalized Experts & 25.1 & 50.2 & 84.4 & 95.7 & 85.9 & 62.5 & 81.9 & 67.2 \\
      w.o. Specialized Experts & 24.0 & 48.8 & 82.8 & 88.7 & 90.6 & 52.1 & 71.4 & 63.1 \\
      w.o. Auxiliary Loss      & 32.2 & 59.7 & 90.0 & 91.8 & 96.8 & 62.9 & 85.4 & 72.0 \\
      w.o. Gradient Scaling    & 32.5 & 59.5 & 90.2 & 91.6 & 96.5 & 63.1 & 85.6 & 72.1 \\
            \midrule
       \textbf{VLA-GSE (Ours)} & \textbf{35.8} & \textbf{61.0} & \textbf{91.0} & \textbf{94.9} & \textbf{98.3} & \textbf{65.2} & \textbf{87.3} & \textbf{74.1} \\
      \bottomrule
    \end{tabular}%。IKD-GSE (Ours)
  }
\end{table*}

\section{Real-World Experiments}

\begin{figure*}[t]
    \centering
    \includegraphics[width=\textwidth]{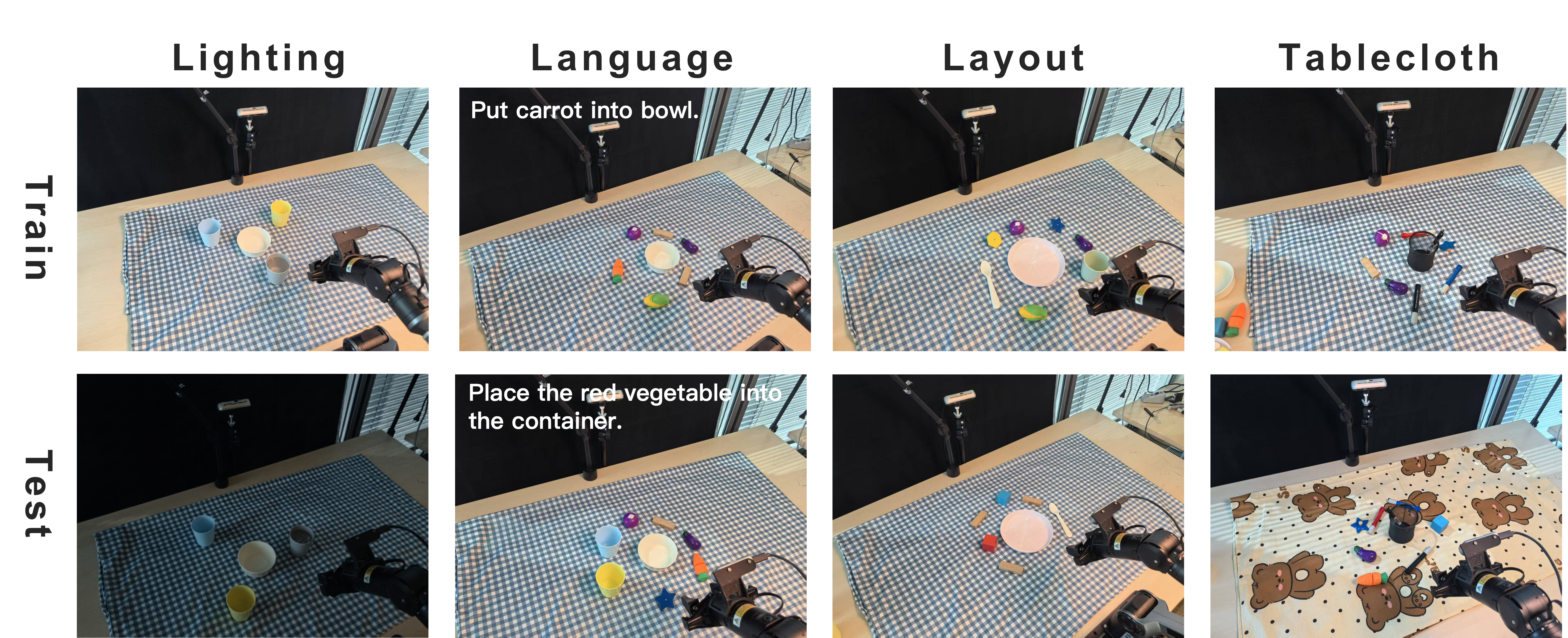
    }
    \caption{\textbf{Real-world generalization settings.}
    The four train-test pairs correspond to variations in \textit{lighting}, \textit{language}, \textit{clutter layout}, and \textit{background tablecloth}, respectively. 
    For \textit{language variations}, the training instruction is \textit{``Put carrot into bowl.''} 
At test time, we replace it with \textit{``Place the red vegetable into the container.''} 
    }
    \label{fig:real_world_settings}
\end{figure*}

\begin{table*}[t]
    \centering
    \caption{Results on \textbf{real-world} generalization. Each task contains 60 evaluation episodes, including 4 OOD settings: variations in \textit{lighting}, \textit{language}, \textit{clutter layout}, and \textit{background tablecloth}. Results are reported in success rate (\%).}
    \label{tab:real_world_results}
    \resizebox{\textwidth}{!}{%
    \begin{tabular}{l|cccc|c}
        \toprule
        Model & Pour water into bowl & Pick spoon into plate & Put carrot into bowl & Take pen out of pen container & Average \\
        \midrule
        OpenVLA-OFT \citep{kim2025fine}    & 38.3 & 41.7 & 48.3 & 43.3 & 42.9 \\
        $\pi_0$ \citep{black2024pi0}             & 51.7 & 55.0 & 60.0 & 56.7 & 55.8 \\
        $\pi_{0.5}$ \citep{physicalintelligence2025pi05} & 73.3 & 76.7 & 75.0 & 71.7 & 74.2 \\
        Qwen3-VL-LoRA \citepalias{starvla2025}   & 58.3 & 61.7 & 66.7 & 60.0 & 61.7 \\
        Qwen3-VL-FFT \citepalias{starvla2025}   & 63.3 & 66.7 & 70.0 & 63.3 & 65.8 \\
        \textbf{VLA-GSE (Ours)} & \textbf{81.7} & \textbf{85.0} & \textbf{83.3} & \textbf{80.0} & \textbf{82.5} \\
        \bottomrule
    \end{tabular}%
    }
\end{table*}

To further validate our framework, we conduct extensive real-world experiments on the AgileX PiPER robot. 
The robot is equipped with a fixed front-view camera and a wrist camera for visual observation.
As illustrated in Figure~\ref{fig:real_world_settings}, the real-world benchmark consists of four tasks: \textit{Pour water into bowl}, \textit{Pick spoon into plate}, \textit{Put carrot into bowl}, and \textit{Take pen out of pen container}. Each task is further evaluated under four types of distribution shifts: lighting changes, language variations, cluttered layout variations, and background tablecloth changes.
For each generalization dimension, we collect 15 test cases, resulting in 60 trials per task. 
% Detailed task description is shown in Appendix~\ref{}.

% and 240 trials in total.

Table~\ref{tab:real_world_results} summarizes the real-world generalization results. VLA-GSE achieves the best performance on all 4 tasks, yielding the highest average success rate of 82.5\%. Among the baselines, $\pi_{0.5}$ is the strongest, followed by Qwen3-VL-FFT, Qwen3-VL-LoRA, and $\pi_0$, while OpenVLA-OFT performs the worst. Notably, Qwen3-VL-LoRA consistently underperforms Qwen3-VL-FFT across all tasks, which aligns with the trend observed in Table~\ref{tab:libero_plus_ft}: parameter-efficient adaptation improves transferability, but FFT remains more effective than vanilla low-rank adaptation for acquiring robust action grounding. By contrast, VLA-GSE delivers a clear margin over all baselines, indicating that its generalized-specialized decomposition is better suited for handling real-world distribution shifts in illumination, background, clutter layout, and instruction phrasing. 
We also compare real-world inference time between VLA-GSE and baselines in Appendix~\ref{app:real_time}.

\section{Conclusion}

We present VLA-GSE, a parameter-efficient framework for adapting pre-trained VLMs to VLA policies that improves control adaptation while retaining the knowledge-preservation benefits of PEFT. VLA-GSE combines a generalized expert (shared adaptation path) with routed specialized experts, together with SVD-based initialization, expert-wise gradient scale balancing, and expectation-based backbone weight adjustment, to increase adaptation capacity for precise robotic control under a fixed trainable-parameter budget. 
Empirically, VLA-GSE consistently outperforms strong FFT and PEFT baselines on simulation and real-world manipulation benchmarks under a comparable trainable-parameter budget, while preserving pre-trained VLM capability comparably to LoRA and outperforming both FFT and co-trained baseline VLA models. Overall, these results position effective PEFT as a promising direction for VLM-to-VLA adaptation and as a practical, compute-efficient alternative to heavy co-training strategies that require substantial vision-language data for knowledge preservation.

%---------------------------------------

%We present VLA-GSE, a parameter-efficient framework for adapting pre-trained VLMs to VLA models, which retains the pretrained VLM’s knowledge and generalization capabilities while enabling accurate and efficient adaptation to downstream robotic tasks. By combining an always-active generalized expert with routed specialized experts, together with SVD-based initialization, expert-wise gradient scaling, expectation-based backbone weight adjustment, and auxiliary load balancing, VLA-GSE provides a principled and efficient solution to the representation degradation problem in VLM-to-VLA transfer. Empirically, it achieves strong gains over both FFT and prior PEFT baselines on LIBERO-Plus, while also transferring effectively to real-world manipulation under multiple distribution shifts.
% Overall, our results indicate that explicitly disentangling generalized knowledge transformation from task-specific control adaptation is a promising direction for building efficient and robust VLA systems.

% \bibliographystyle{Ref}  %plainnat,abbrvnat,unsrtnat
\bibliographystyle{plainnat}
\small
\bibliography{main}
\normalsize

\newpage
\appendix

\section{Technical Appendices and Supplementary Material}

\subsection{Limitations and Future Works}

A limitation of the current study is that our evaluation is still restricted to robotics-oriented settings, including embodied policy learning and real-world manipulation tasks, and we have not yet validated the proposed framework in other multimodal domains that also require high precision, such as medical image analysis, financial market forecasting, or ancient literature restoration.
Extending VLA-GSE to these also high-precision and multimodal domains is an important direction for future works, and would help clarify the generality of the generalized–specialized expert design beyond robotics. Despite this limitation, our main contribution is to present a simple, parameter-efficient, and effective adaptation framework that improves robustness and transfer in robotic VLA learning.

\subsection{Statement for Use of LLMs}

LLMs were only used to assist with language polishing in certain sections of this paper.

\subsection{Reproducibility Statement}

We provide open-source code to reproduce all experiments.
To facilitate reproducibility, we additionally include a minimal working example and environment setup instructions, enabling researchers to verify correctness and replicate our reported results with minimal effort.
% The complete anonymous repository is available at \url{https://anonymous.4open.science/r/VLA-GSE}.
Code is available at \url{https://github.com/YuhuaJiang2002/VLA-GSE}.

\subsection{Impact Statement}

This paper presents work whose goal is to advance the field of machine 
learning. 
We do not identify any specific impacts of this work that require particular emphasis here.

\section{Proof of Theorem~\ref{thm:scaling_factor}}
\label{sec:proof_scaling}

\begin{proof}
For the $i$-th specialized expert, its effective contribution to the aggregated
equivalent weight is
$w^i(x) W_s^i$, where
$W_s^i = s_s^i B_s^i A_s^i$.
Let
$g = \frac{\partial \mathcal{L}_{\text{final}}}{\partial W_{eq}(x)}$
denote the gradient of the loss with respect to the aggregated equivalent weight.
By the chain rule, the localized gradients with respect to $A_s^i$ and $B_s^i$ are
\begin{align}
g_A^i
&=
\frac{\partial \mathcal{L}_{\text{final}}}{\partial A_s^i}
=
\frac{\partial \bigl(w^i(x) W_s^i\bigr)}{\partial A_s^i}
\frac{\partial \mathcal{L}_{\text{final}}}{\partial \bigl(w^i(x) W_s^i\bigr)}
=
w^i(x)\, s_s^i (B_s^i)^\top g,
\label{eq:grad_A_appendix}
\\
g_B^i
&=
\frac{\partial \mathcal{L}_{\text{final}}}{\partial B_s^i}
=
\frac{\partial \mathcal{L}_{\text{final}}}{\partial \bigl(w^i(x) W_s^i\bigr)}
\frac{\partial \bigl(w^i(x) W_s^i\bigr)}{\partial B_s^i}
=
w^i(x)\, s_s^i g (A_s^i)^\top.
\label{eq:grad_B_appendix}
\end{align}

Using the initialization
\begin{equation}
B_s^i = \sqrt{\frac{1}{s_s^i}}\, U_i \Sigma_i^{1/2},
\qquad
A_s^i = \sqrt{\frac{1}{s_s^i}}\, \Sigma_i^{1/2} V_i^\top,
\label{eq:init_AB_appendix}
\end{equation}
we obtain
\begin{align}
g_A^i
&=
w^i(x)\, s_s^i
\left(
\sqrt{\frac{1}{s_s^i}}\, U_i \Sigma_i^{1/2}
\right)^\top g
=
w^i(x)\, \sqrt{s_s^i}\, \Sigma_i^{1/2} U_i^\top g,
\label{eq:localized_grad_A_appendix}
\\
g_B^i
&=
w^i(x)\, s_s^i
g
\left(
\sqrt{\frac{1}{s_s^i}}\, \Sigma_i^{1/2} V_i^\top
\right)^\top
=
w^i(x)\, \sqrt{s_s^i}\, g V_i \Sigma_i^{1/2}.
\label{eq:localized_grad_B_appendix}
\end{align}

We first analyze the squared Frobenius norm of $g_A^i$:
\begin{align}
\|g_A^i\|_F^2
&=
\mathrm{Tr}\!\left((g_A^i)^\top g_A^i\right)
\nonumber\\
&=
\mathrm{Tr}\!\left(
\left(w^i(x)\, \sqrt{s_s^i}\, \Sigma_i^{1/2} U_i^\top g\right)^\top
\left(w^i(x)\, \sqrt{s_s^i}\, \Sigma_i^{1/2} U_i^\top g\right)
\right)
\nonumber\\
&=
(w^i(x))^2 s_s^i \,
\mathrm{Tr}\!\left(
g^\top U_i \Sigma_i U_i^\top g
\right).
\label{eq:frob_A_appendix}
\end{align}
Taking expectation over the training batch distribution and using the cyclic property of the trace gives
\begin{align}
\mathbb{E}\!\left[\|g_A^i\|_F^2\right]
&=
s_s^i \,
\mathbb{E}\!\left[
(w^i(x))^2
\mathrm{Tr}\!\left(
g^\top U_i \Sigma_i U_i^\top g
\right)
\right].
\label{eq:EA_appendix_raw}
\end{align}
To analyze the expected squared gradient norm, we consider the idealized
balanced-routing condition and define
\begin{equation}
\rho_w := \mathbb{E}\!\left[(w^i(x))^2\right].
\label{eq:rho_w_appendix}
\end{equation}
Under balanced routing, $\rho_w$ is identical across experts and thus does not
depend on $i$. Therefore, when taking expectations of
$\|g_A^i\|_F^2$ or $\|g_B^i\|_F^2$, the routing term contributes only the same
multiplicative factor $\rho_w$ for all experts. As a result, it does not affect
the relative scaling across experts, and the trace-inverse scaling rule remains
unchanged.
Absorbing this shared routing
second-moment factor, we obtain
\begin{align}
\mathbb{E}\!\left[\|g_A^i\|_F^2\right]
&=
\rho_w \, s_s^i \,
\mathrm{Tr}\!\left(
U_i \Sigma_i U_i^\top \, \mathbb{E}[g g^\top]
\right)
\nonumber\\
&=
\rho_w \, s_s^i \,
\mathrm{Tr}\!\left(
\Sigma_i U_i^\top G_L U_i
\right).
\label{eq:EA_appendix}
\end{align}
Define
\begin{equation}
\alpha_i^L = \mathrm{Tr}\!\left(\Sigma_i U_i^\top G_L U_i\right).
\label{eq:alphaL_appendix}
\end{equation}
Then Eq.~\eqref{eq:EA_appendix} becomes
\begin{equation}
\mathbb{E}\!\left[\|g_A^i\|_F^2\right]
=
\rho_w \, s_s^i \, \alpha_i^L.
\label{eq:EA_compact_appendix}
\end{equation}

We next analyze $g_B^i$. From Eq.~\eqref{eq:localized_grad_B_appendix},
\begin{align}
\|g_B^i\|_F^2
&=
\mathrm{Tr}\!\left((g_B^i)^\top g_B^i\right)
\nonumber\\
&=
\mathrm{Tr}\!\left(
\left(w^i(x)\, \sqrt{s_s^i}\, g V_i \Sigma_i^{1/2}\right)^\top
\left(w^i(x)\, \sqrt{s_s^i}\, g V_i \Sigma_i^{1/2}\right)
\right)
\nonumber\\
&=
(w^i(x))^2 s_s^i \,
\mathrm{Tr}\!\left(
\Sigma_i V_i^\top g^\top g V_i
\right).
\label{eq:frob_B_appendix}
\end{align}
Taking expectation yields
\begin{align}
\mathbb{E}\!\left[\|g_B^i\|_F^2\right]
&=
s_s^i \,
\mathbb{E}\!\left[
(w^i(x))^2
\mathrm{Tr}\!\left(
\Sigma_i V_i^\top g^\top g V_i
\right)
\right]
\nonumber\\
&=
\rho_w \, s_s^i \,
\mathrm{Tr}\!\left(
\Sigma_i V_i^\top \mathbb{E}[g^\top g] V_i
\right)
\nonumber\\
&=
\rho_w \, s_s^i \,
\mathrm{Tr}\!\left(
\Sigma_i V_i^\top G_R V_i
\right).
\label{eq:EB_appendix}
\end{align}
Define
\begin{equation}
\alpha_i^R = \mathrm{Tr}\!\left(\Sigma_i V_i^\top G_R V_i\right).
\label{eq:alphaR_appendix}
\end{equation}
Then Eq.~\eqref{eq:EB_appendix} becomes
\begin{equation}
\mathbb{E}\!\left[\|g_B^i\|_F^2\right]
=
\rho_w \, s_s^i \, \alpha_i^R.
\label{eq:EB_compact_appendix}
\end{equation}

Equations~\eqref{eq:EA_compact_appendix} and~\eqref{eq:EB_compact_appendix}
prove the first part of the theorem under the idealized balanced-routing
assumption at initialization. Since $\rho_w$ is shared across experts, it does
not affect the relative scaling across experts, and thus does not change the
trace-inverse scaling rule derived below.

We now prove the trace-inverse scaling rule.
Assume there exist constants $\kappa_L,\kappa_R > 0$ such that\footnote{This assumption is automatically satisfied under $G_L =\mathbb{E}[g g^\top] \propto I$ and $G_R =\mathbb{E}[g^\top g] \propto I$, which is the commonly used isotropic gradient variance assumption \citep[used in Theorem 1 of][]{jastrzebski2018three}.}
\begin{equation}
\alpha_i^L = \kappa_L \, \mathrm{Tr}(\Sigma_i),
\qquad
\alpha_i^R = \kappa_R \, \mathrm{Tr}(\Sigma_i),
\qquad
\forall i \in \{1,\dots,E\}.
\label{eq:projected_uniformity_appendix}
\end{equation}
Substituting Eq.~\eqref{eq:projected_uniformity_appendix} into
Eqs.~\eqref{eq:EA_compact_appendix} and~\eqref{eq:EB_compact_appendix}, we obtain
\begin{align}
\mathbb{E}\!\left[\|g_A^i\|_F^2\right]
&=
\rho_w s_s^i \kappa_L \, \mathrm{Tr}(\Sigma_i),
\\
\mathbb{E}\!\left[\|g_B^i\|_F^2\right]
&=
\rho_w s_s^i \kappa_R \, \mathrm{Tr}(\Sigma_i).
\end{align}
Therefore, if we choose
\begin{equation}
s_s^i
=
s_{\mathrm{base}} \cdot \frac{C}{\mathrm{Tr}(\Sigma_i)},
\label{eq:trace_inverse_scaling_appendix}
\end{equation}
then it follows that
\begin{align}
\mathbb{E}\!\left[\|g_A^i\|_F^2\right]
&=
\rho_w s_{\mathrm{base}} C \kappa_L,
\\
\mathbb{E}\!\left[\|g_B^i\|_F^2\right]
&=
\rho_w s_{\mathrm{base}} C \kappa_R,
\end{align}
which are independent of $i$.
Hence, the expected squared Frobenius norms of the localized gradients are equalized across all specialized experts.
This proves the theorem.
\end{proof}

\begin{proposition}[Isotropic global gradient variance implies projected uniformity]
\label{prop:isotropic_projected_uniformity}
A sufficient condition for the projected-uniformity condition in
Eq.~\eqref{eq:projected_uniformity_appendix} is that the global gradient
second moments are isotropic, namely
\begin{equation}
G_L = \mathbb{E}[g g^\top] = c_L I,
\qquad
G_R = \mathbb{E}[g^\top g] = c_R I,
\end{equation}
for some constants $c_L, c_R > 0$.
Under this assumption, for every specialized expert $i$,
\begin{equation}
\alpha_i^L = c_L \, \mathrm{Tr}(\Sigma_i),
\qquad
\alpha_i^R = c_R \, \mathrm{Tr}(\Sigma_i),
\end{equation}
where
\begin{equation}
\alpha_i^L = \mathrm{Tr}\!\left(\Sigma_i U_i^\top G_L U_i\right),
\qquad
\alpha_i^R = \mathrm{Tr}\!\left(\Sigma_i V_i^\top G_R V_i\right).
\end{equation}
Therefore, Eq.~\eqref{eq:projected_uniformity_appendix} holds with
$\kappa_L = c_L$ and $\kappa_R = c_R$.
\end{proposition}

\begin{proof}
Since $U_i$ and $V_i$ are composed of orthogonal singular vectors, they satisfy
$U_i^\top U_i = I$ and $V_i^\top V_i = I$.
Substituting $G_L = c_L I$ into $\alpha_i^L$ gives
\begin{equation}
\alpha_i^L
=
\mathrm{Tr}\!\left(\Sigma_i U_i^\top (c_L I) U_i\right)
=
c_L \, \mathrm{Tr}\!\left(\Sigma_i U_i^\top U_i\right)
=
c_L \, \mathrm{Tr}(\Sigma_i).
\end{equation}
Similarly, substituting $G_R = c_R I$ into $\alpha_i^R$ yields
\begin{equation}
\alpha_i^R
=
\mathrm{Tr}\!\left(\Sigma_i V_i^\top (c_R I) V_i\right)
=
c_R \, \mathrm{Tr}\!\left(\Sigma_i V_i^\top V_i\right)
=
c_R \, \mathrm{Tr}(\Sigma_i).
\end{equation}
Hence the projected-uniformity condition is satisfied exactly with
$\kappa_L = c_L$ and $\kappa_R = c_R$.
This isotropic case is a standard analytically tractable setting in the
optimization literature for studying gradient-based methods
\citep[see Theorem 1 of][]{jastrzebski2018three}.
\end{proof}

\section{Proof of Expectation-Based Backbone Weight Adjustment}
\label{proof2}

\begin{proof}
Recall that the effective weight of the proposed module is
\begin{equation}
W_{eq}(x)
=
\tilde{W}_0
+
s_g B_g A_g
+
\sum_{i=1}^{E} w^i(x)\, s_s^i B_s^i A_s^i,
\end{equation}
where $\tilde{W}_0$ is defined as
\begin{equation}
\tilde{W}_0
=
W_0
-
s_g B_g A_g
-
\frac{1}{E}\sum_{i=1}^{E} s_s^i B_s^i A_s^i.
\end{equation}

Under the uniform-routing assumption at initialization, each specialized expert is selected with equal probability in expectation. Hence,
\begin{equation}
\mathbb{E}_x[w^i(x)] = \frac{1}{E},
\qquad i=1,\dots,E.
\end{equation}
Therefore, taking expectation over the input-dependent routing gives
\begin{equation}
\begin{aligned}
\mathbb{E}_x[W_{eq}(x)]
&=
\tilde{W}_0
+
s_g B_g A_g
+
\mathbb{E}_x\!\left[\sum_{i=1}^{E} w^i(x)\, s_s^i B_s^i A_s^i\right] \\
&=
\tilde{W}_0
+
s_g B_g A_g
+
\sum_{i=1}^{E} \mathbb{E}_x[w^i(x)]\, s_s^i B_s^i A_s^i \\
&=
\tilde{W}_0
+
s_g B_g A_g
+
\frac{1}{E}\sum_{i=1}^{E} s_s^i B_s^i A_s^i.
\end{aligned}
\end{equation}
Substituting the definition of $\tilde{W}_0$ yields
\begin{equation}
\mathbb{E}_x[W_{eq}(x)] = W_0.
\end{equation}
Thus, the effective weight $W_{eq}(x)$ is aligned with the original backbone weight $W_0$ in expectation at initialization.
\end{proof}

\section{More Experiment Results}

\subsection{More Qualitative Evaluation in Generalization Scenarios}

\begin{figure}[htbp]
    \centering
    \includegraphics[width=\textwidth]{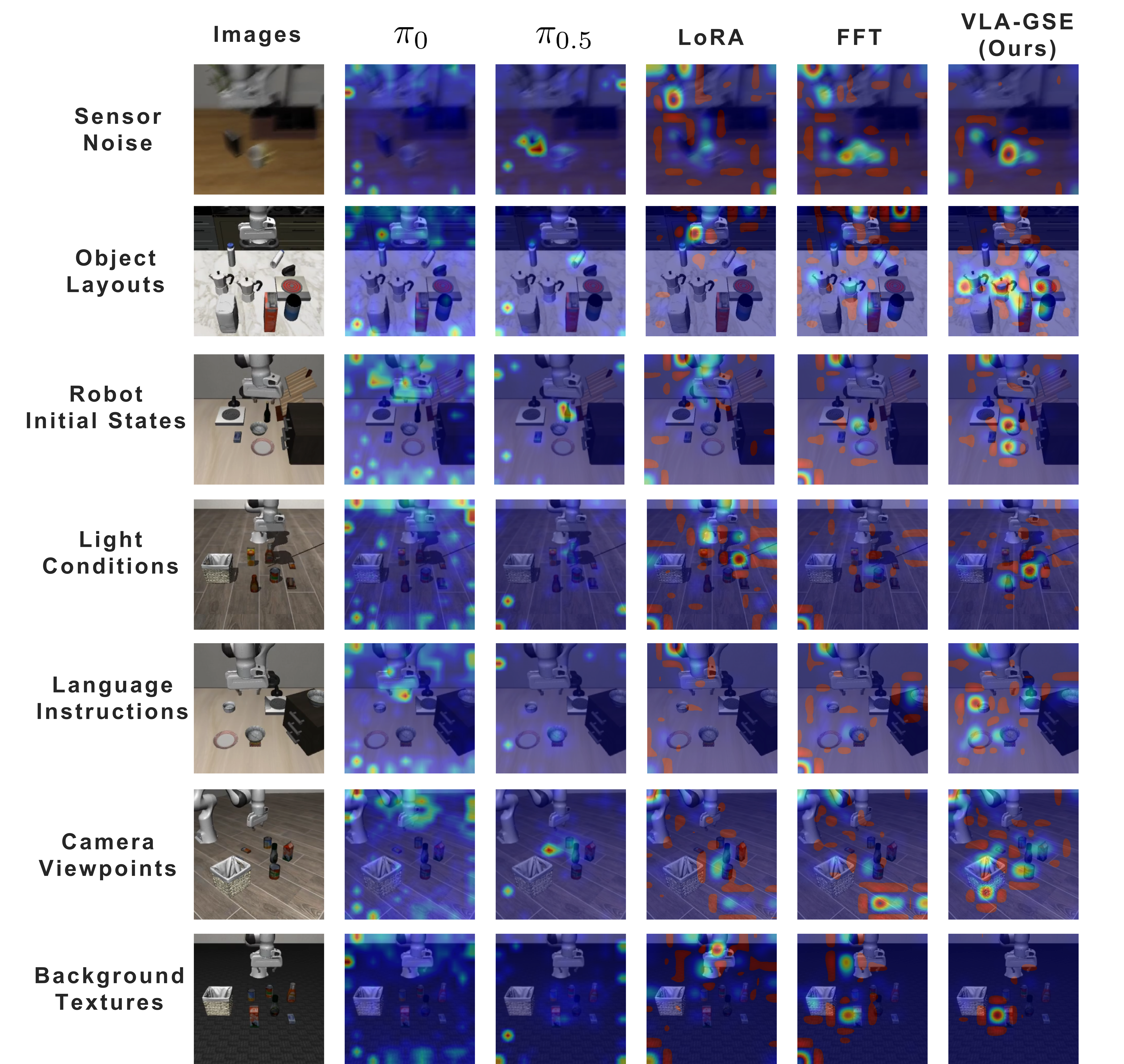} 
    \caption{Extended attention map visualizations under seven diverse generalization scenarios. As a comprehensive supplement to the main text, this figure demonstrates the robustness of our VLA-GSE method across variations in Camera Viewpoints, Object Layouts, Robot Initial States, Light Conditions, Language Instructions, Sensor Noise, and Background Textures.
    All attention maps are taken from the last transformer layer of the models. 
    GSE consistently maintains precise attention on the graspable edges of target objects, while baseline models ($\pi_{0}$, $\pi_{0.5}$, LoRA, and FFT) suffer from severe visual distractions or become unfocused under domain shifts.}
    \label{fig:attention_map_appendix}
\end{figure}

As a comprehensive supplement to the qualitative evaluation presented in the main text, Figure~\ref{fig:attention_map_appendix} provides an extensive visualization of model attention maps across a wider range of out-of-distribution (OOD) generalization scenarios. Specifically, we evaluate the visual grounding robustness of all models under seven challenging axes of variation: Camera Viewpoints, Object Layouts, Robot Initial States, Light Conditions, Language Instructions, Sensor Noise, and Background Textures. 
All attention maps are taken from the last transformer layer of the models.

The results in Figure~\ref{fig:attention_map_appendix} consistently corroborate our previous findings and further highlight the exceptional generalization capabilities of our proposed GSE method. Across all seven diverse visual domain shifts, GSE stably and precisely focuses its attention on the relevant target objects and their critical graspable edges. This remarkable robustness empirically validates our core motivation: by effectively preserving the broad, pre-trained visual knowledge of the base Qwen3-VL model, GSE inherits its strong invariance to environmental variations (such as lighting changes, background shifts, and sensor noise). Concurrently, the efficiently learned manipulation skills allow GSE to accurately pinpoint the optimal affordance regions regardless of the scene layout or camera angle. 

Conversely, the baseline models exhibit significant fragility when faced with these generalization scenarios. $\pi_{0}$ consistently demonstrates scattered and uninformative attention distributions across all settings. While $\pi_{0.5}$, LoRA, and FFT show localized attention in some specific cases, their performance drastically degrades under severe domain shifts (e.g., sensor noise or complex background textures). In these challenging settings, the baselines either completely lose track of the objects or severely overfit to spurious background features and incorrect distractors. This widespread failure among the standard fine-tuning and basic PEFT baselines strongly implies that they suffer from catastrophic forgetting, losing the base VLM's inherent visual robustness during the adaptation to robotic tasks. In contrast, GSE successfully bridges the gap between robust, general-purpose visual understanding and precise, task-specific action execution.

% \begin{table}[t]
%     \centering
%     \caption{Real-world inference time comparison on an NVIDIA RTX 5080 GPU. We report the wall-clock time for generating one action chunk (batch size 1). Lower is better.}
%     \label{tab:real_time_comparison}
%     \begin{tabular}{lc}
%         \toprule
%         Model & Time per action chunk (ms) \\
%         \midrule
%         Qwen3-VL-FFT & \textbf{88.9} $\pm$ 1.8 \\
%         $\pi_0$ & 122.8 $\pm$ 2.6 \\
%         $\pi_{0.5}$ & 129.5 $\pm$ 2.7 \\
%         OpenVLA-OFT & 96.6 $\pm$ 2.1 \\
%         \textbf{VLA-GSE (Ours)} & 94.7 $\pm$ 2.0 \\
%         \bottomrule
%     \end{tabular}
% \end{table}

\subsection{Comparison of VLA Models on LIBERO Results}

\begin{table}[t]
\centering
\caption{LIBERO simulation benchmark results. Success rates (SR, \%) across the four LIBERO task suites.}
\label{tab:libero_sim}
\begin{tabular}{l|cccc|c}
\toprule
Method & Spatial & Object & Goal & Long & Average \\
\midrule
Diffusion Policy~\citep{Chi-RSS-23} & 78.3 & 92.5 & 68.3 & 50.5 & 72.4 \\
Dita~\citep{Hou_2025_ICCV} & 97.4 & 94.8 & 93.2 & 83.6 & 92.3 \\
$\pi_0$~\citep{BlackK-RSS-25} & 96.8 & 98.8 & 95.8 & 85.2 & 94.2 \\
$\pi_{0.5}$~\citep{physicalintelligence2025pi05} & 98.8 & 98.2 & 98.0 & 92.4 & 96.9 \\
FLOWER~\citep{pmlr-v305-reuss25a} & 97.5 & 99.1 & 96.1 & 94.9 & 96.9 \\
GR00T-N1.6~\citep{nvidia2025grootn16} & 97.7 & 98.5 & 97.5 & 94.4 & 97.0 \\
OpenVLA-OFT~\citep{kim2025fine} & 97.6 & 98.4 & 97.9 & 94.5 & 97.1 \\
CogVLA~\citep{li2025cogvla} & 98.6 & 98.8 & 96.6 & 95.4 & 97.4 \\
VLANeXt~\citep{wu2026vlanext} & \textbf{99.0} & 99.2 & 96.6 & 94.6 & 97.4 \\
X-VLA~\citep{zheng2025x} & 98.2 & 98.6 & 97.8 & \textbf{97.6} & 98.1 \\
\textbf{VLA-GSE (Ours)} & 98.8 & \textbf{99.8} & \textbf{98.2} & 96.8 & \textbf{98.4} \\
\bottomrule
\end{tabular}
\end{table}

Table~\ref{tab:libero_sim} reports the simulation results on LIBERO. VLA-GSE achieves the best overall performance, reaching an average success rate of 98.4\%. Even after including strong recent baselines such as FLOWER, GR00T-N1.6, VLANeXt, and X-VLA, our method still ranks first overall, outperforming the strongest baseline, X-VLA, by 0.3 points in average success rate. VLA-GSE also establishes new best results on Object (99.8\%) and Goal (98.2\%), while remaining highly competitive on Spatial (98.8\%) and Long (96.8\%).

These results suggest that GSE provides a strong balance between fine-grained task adaptation and robustness across heterogeneous task suites. In particular, compared with X-VLA, the gains on Object and Goal are +1.2 and +0.4 points, respectively, indicating that the proposed generalized-specialized expert decomposition improves task-specific adaptation without sacrificing cross-suite generalization. Although VLANeXt attains the best Spatial score and X-VLA achieves the best Long score, VLA-GSE delivers the highest average performance, demonstrating the strongest overall trade-off across the four LIBERO suites. Overall, the results show that parameter-efficient adaptation with generalized and specialized experts can remain competitive with, and often surpass, strong recent VLA baselines without requiring full-model fine-tuning.

\subsection{Comparison of Fine-Tuning Methods on LIBERO Results}

\begin{table}[t]
\centering
\caption{Fine-tuning method comparison on the standard LIBERO benchmark. Results are reported as success rate (\%) across the four task suites. All PEFT methods are configured with comparable trainable parameter budgets, following the setting in Table~\ref{tab:libero_plus_ft}.}
\label{tab:libero_ft}
\begin{tabular}{l|c|cccc|c}
\toprule
Method & Params (\%) & Spatial & Object & Goal & Long & Average \\
\midrule
FFT (Full Fine-Tuning)                                & 100.00 & 97.8 & 98.8 & 96.4 & 94.2 & 96.8 \\
LoRA~\citep{hu2022lora}                & 2.58 & 94.4 & 96.8 & 90.4 & 78.4 & 90.0 \\
MiLoRA~\citep{wang2025milora}          & 2.52 & 97.0 & 98.2 & 95.0 & 84.2 & 93.6 \\
PiSSA~\citep{meng2024pissa}            & 2.54 & 96.0 & 97.4 & 93.6 & 89.8 & 94.2 \\
rsLoRA~\citep{kalajdzievski2023rslora} & 2.52 & 96.4 & 98.0 & 94.0 & 90.8 & 94.8 \\
HydraLoRA~\citep{tian2024hydralora}    & 2.55 & 96.6 & 98.0 & 94.2 & 91.2 & 95.0 \\
DoRA~\citep{liu2024dora}               & 2.54 & 96.6 & 98.2 & 94.4 & 91.6 & 95.2 \\
AdaMoLE~\citep{liu2024adamole}         & 2.53 & 96.8 & 98.2 & 94.8 & 91.8 & 95.4 \\
MoLoRA~\citep{zadouri2024pushing}      & 2.56 & 97.2 & 98.4 & 95.2 & 92.4 & 95.8 \\
GOAT~\citep{fan2025goat}               & 2.54 & 97.4 & 98.6 & 95.8 & 93.0 & 96.2 \\
\textbf{VLA-GSE (Ours)}                & 2.51 & \textbf{98.8} & \textbf{99.8} & \textbf{98.2} & \textbf{96.8} & \textbf{98.4} \\
\bottomrule
\end{tabular}
\end{table}

Table~\ref{tab:libero_ft} reports the results on the standard LIBERO benchmark. 
Overall, the ranking is broadly consistent with that on LIBERO-Plus: stronger parameter-efficient fine-tuning methods on zero-shot generalization also tend to achieve better performance on the in-distribution LIBERO benchmark. 
Vanilla low-rank adaptation, such as LoRA, already provides a competitive baseline, while more advanced variants, including DoRA, AdaMoLE, MoLoRA, and GOAT, further improve the average success rate, indicating that a richer adaptation space is beneficial even under the standard simulation setting.

Among all baselines, FFT remains highly competitive, achieving an average success rate of 96.8\%. 
Nevertheless, VLA-GSE still delivers the strongest overall performance, reaching {98.4\%} average success rate and ranking first on all four suites. 
Compared with FFT, our method improves Spatial from 97.8\% to {98.8\%}, Object from 98.8\% to {99.8\%}, Goal from 96.4\% to {98.2\%}, and Long from 94.2\% to {96.8\%}, corresponding to gains of {1.0}, {1.0}, {1.8}, and {2.6} points, respectively. 
The largest improvement is observed on the Long suite, suggesting that VLA-GSE is particularly effective for long-horizon manipulation requiring sustained temporal consistency and accurate multi-step action composition.

Compared with the strongest PEFT baseline, GOAT, VLA-GSE improves the average success rate from 96.2\% to {98.4\%}, yielding a margin of {2.2} points. 
This advantage is again most pronounced on the more challenging Goal and Long suites, where our method outperforms GOAT by {2.4} and {3.8} points, respectively. 
Notably, Object is already close to saturation for most competitive methods, whereas Long still exhibits a clear performance spread. 
This pattern indicates that long-horizon sequential manipulation remains the main differentiating factor once short-horizon perception and control are well handled. 
Overall, these results provide additional evidence that explicitly combining generalized transferable adaptation with specialized experts yields a more effective fine-tuning mechanism for VLAs than both conventional LoRA-style methods and existing LoRA-MoE variants.

\subsection{Real-World Inference Time Comparison}
\label{app:real_time}

\begin{figure}[t]
    \centering
    \includegraphics[width=0.75\linewidth]{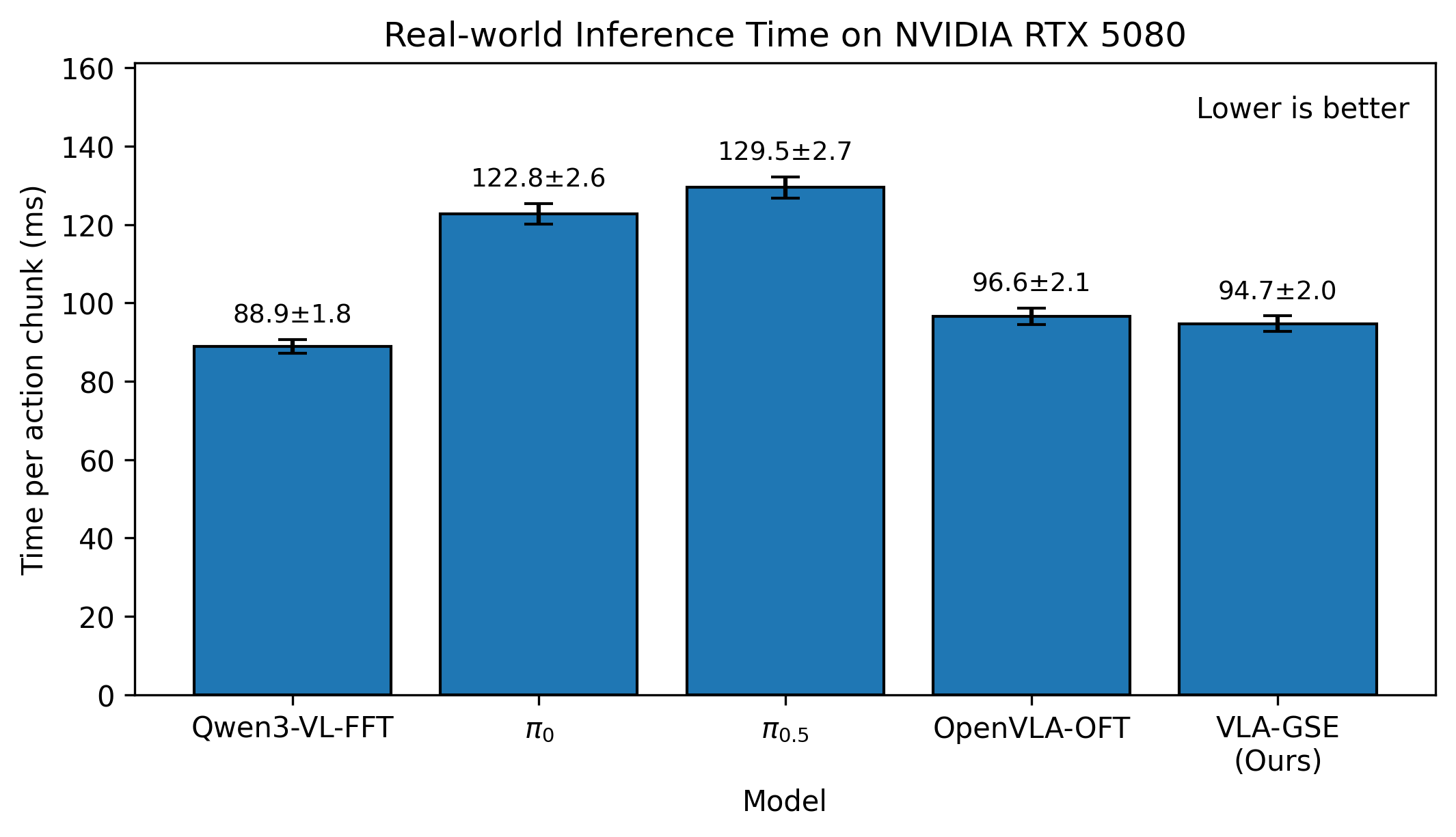}
    \caption{Real-world inference latency comparison on an NVIDIA RTX 5080 GPU. We report the wall-clock time for generating one action chunk with batch size 1. Error bars denote standard deviation over 120 runs. Lower is better.}
    \label{fig:real_time_comparison}
\end{figure}

We further compare the real-world inference efficiency of different policies by measuring the wall-clock latency required to generate one action chunk. All methods are evaluated with batch size 1 under the same deployment pipeline on the same real-robot workstation equipped with an NVIDIA RTX 5080 GPU. We measure the elapsed time from receiving the current observation to returning the full action chunk, and report the average latency over 120 runs.

As shown in Figure~\ref{fig:real_time_comparison}, Qwen3-VL-FFT achieves the lowest latency, while VLA-GSE is the second fastest method among the compared policies. The most controlled comparison is between VLA-GSE and Qwen3-VL-FFT, since both are built on the same Qwen3-VL-4B backbone and mainly differ in the adaptation strategy. Under this matched setting, the latency gap indicates that the proposed generalized-specialized expert design introduces only a modest additional inference overhead.

We note, however, that absolute latency comparisons across different policy families should be interpreted with caution. In particular, the lower latency of VLA-GSE relative to $\pi_0$ and $\pi_{0.5}$ is not solely due to the proposed adapter design, but also reflects differences in the action-generation paradigm: VLA-GSE adopts OFT-style parallel decoding, whereas $\pi_0$ and $\pi_{0.5}$ rely on flow-matching-based generation, which is inherently more expensive at inference time. Similarly, the lower latency of VLA-GSE relative to OpenVLA-OFT should not be over-interpreted as a pure architectural advantage of GSE, since OpenVLA-OFT is built on a larger 7B backbone, whereas VLA-GSE uses a 4B backbone. Therefore, these cross-model comparisons are best viewed as end-to-end deployment measurements rather than clean isolation of module-level efficiency.

Overall, the main takeaway from Figure~\ref{fig:real_time_comparison} is that, relative to the matched Qwen3-VL-FFT baseline, VLA-GSE improves policy performance while incurring only a small additional inference-time cost, while remaining practically efficient for real-world deployment.

\subsection{Suite-wise Results on LIBERO-Plus}
\label{suitewise}

For compact comparison, tables in the main paper reports only the averaged LIBERO-Plus performance of {VLA-GSE}. Here, we provide its detailed zero-shot results on the four LIBERO suites: \textit{Spatial}, \textit{Object}, \textit{Goal}, and \textit{Long}.

As shown in Table~\ref{tab:libero_plus_suitewise_vlagse}, {VLA-GSE} achieves strong performance across all four suites, with the best overall results on \textit{Spatial} (90.3) and competitive performance on the more challenging \textit{Goal} and \textit{Long} suites. Notably, the suite-wise breakdown shows that the strong averaged performance in the main paper is not dominated by a single suite, but is supported by consistently high robustness across diverse distribution shifts, including camera, robot embodiment, language, lighting, background, noise, and layout variations. These results further demonstrate that {VLA-GSE} generalizes well across different task structures and evaluation conditions in LIBERO-Plus.

\begin{table*}[t]
  \centering
  \caption{Suite-wise zero-shot performance of {VLA-GSE} on LIBERO-Plus. Results are shown in success rate (\%).}
  \label{tab:libero_plus_suitewise_vlagse}
  \resizebox{\textwidth}{!}{%
    \begin{tabular}{l|ccccccc|c}
      \toprule
      Suite & Camera & Robot & Language & Light & Background & Noise & Layout & Total \\
      \midrule
      Spatial & 89.5 & 80.6 & 89.6 & 94.8 & 96.4 & 89.2 & 95.2 & 90.3 \\
      Object  & 80.5 & 65.3 & 95.1 & 100.0 & 98.4 & 86.6 & 85.1 & 86.2 \\
      Goal    & 51.7 & 67.2 & 79.3 & 99.6 & 96.1 & 76.5 & 62.8 & 74.2 \\
      Long    & 35.8 & 61.0 & 91.0 & 94.9 & 98.3 & 65.2 & 87.3 & 74.1 \\
      \midrule
      Average & 64.4 & {68.5} & {88.8} & {97.3} & {97.3} & 79.4 & {82.6} &{81.2} \\
      \bottomrule
    \end{tabular}
  }
\end{table*}

\subsection{Training Details, Hyperparameters, and Task Description}
\label{sec:appendix_hyperparameters}

\paragraph{Training Details and Hyperparameters}
Table~\ref{tab:hyperparameters} details the comprehensive set of hyperparameters and training configurations used for fine-tuning the Qwen3-VL-4B-Instruct VLM with our proposed VLA-GSE framework. 

\paragraph{Real-World Task Description and Hyperparameters}
For the real-world experiments, we evaluate four representative long-horizon manipulation tasks: \textit{pour water into bowl}, which requires stable container grasping and controlled pouring; \textit{pick spoon into plate}, which involves precise grasping and object placement; \textit{put carrot into bowl}, which tests visually grounded pick-and-place under object-level spatial variation; and \textit{take pen out of pen container}, which requires extracting a slender pen from a cylindrical container. 
For each task, we collect 70 in-distribution demonstrations via teleoperation. 
Each task contains 60 evaluation episodes, covering four out-of-distribution (OOD) generalization settings: variations in \textit{lighting}, which change scene illumination conditions; \textit{language}, which rephrase the task instruction; \textit{clutter layout}, which modify the arrangement of surrounding objects; and \textit{background tablecloth}, which alter the background appearance of the workspace. These settings are designed to evaluate the robustness of the learned policy under realistic visual and linguistic distribution shifts.

During training, we set the learning rate for the VLM and GSE parameters to $1 \times 10^{-5}$, and the learning rate for the action MLP head to $1 \times 10^{-4}$.
The policy for \textit{pour water into bowl} is trained for 8000 iterations, whereas the policies for the other three tasks are each trained for 6000 iterations.
Other real-world task training details and hyperparameters are consistent with Table~\ref{tab:hyperparameters}.

\begin{table}[t]
\centering
\caption{Hyperparameters and Configuration for VLA-GSE Fine-tuning}
\label{tab:hyperparameters}
\begin{tabular}{lc}
\toprule
\textbf{Hyperparameter / Setting} & \textbf{Value} \\
\midrule
\multicolumn{2}{c}{\textit{Architecture \& VLA-GSE Configuration}} \\
\midrule
Base VLM Backbone & Qwen3-VL-4B-Instruct \\
Action Head & OpenVLA-OFT MLP \citep{kim2025fine} \\
Rank ($r$) & 16 \\
Rank of Generalized Expert ($r_g$) & 2 \\
Rank of Specialized Expert ($d$) & 2 \\
Total Experts per GSE Block & 8 \\
Generalized Experts & 1 \\
Top-$k$ Routing & 2 \\
Router Initialization & Random Gaussian Noise \\
Initialization Type &  SVD-based GSE \\
Total Generalized Experts in Model & 356  \\
Total Specialized Experts in Model & 2492 \\
\midrule
\multicolumn{2}{c}{\textit{Training \& Optimization}} \\
\midrule
Training Dataset & LIBERO Benchmark (all 4 suites) \\
Total Optimization Steps & 80,000 \\
Batch Size per GPU & 16 \\
Total Effective Batch Size & 128 \\
Learning Rate (VLM / GSE Params) & $1 \times 10^{-5}$ \\
Learning Rate (Action MLP Head) & $1 \times 10^{-4}$ \\
Auxiliary Loss Weight ($\alpha$) & 0.01 \\
Generalized Scaling Factor ($s_g$) & 2 \\
Specialized Scaling Factor ($s_s^i$) &  $\frac{2 \cdot \sum_{j=1}^E \mathrm{Tr}(\Sigma_j)}{E \cdot \mathrm{Tr}(\Sigma_i)}$ \\
Hardware & 8 $\times$ NVIDIA A100 GPUs \\
\midrule
\multicolumn{2}{c}{\textit{Parameter Efficiency}} \\
\midrule
Total Parameters & 4,551.85 M \\
Frozen Parameters (Non-GSE) & 4,437.82 M \\
Trainable Parameters (Total) & 114.04 M (2.51\%) \\
Trainable Parameters (GSE Modules) & 48.41 M \\
Trainable Parameters (Action Head) & 65.62 M \\
\bottomrule
\end{tabular}
\end{table}

\section{VLA-GSE Algorithm Pseudo-code}
\label{sec:appendix_algorithm}

In this section, we provide the complete pseudo-code for the VLA-GSE framework.
Algorithm~\ref{alg:gse} explicitly outlines the two primary operational phases of our approach. 
The first phase is the SVD-based adaptive initialization, where the pre-trained weight matrix is decomposed. The dominant singular values are allocated to the generalized expert to anchor foundational knowledge, while the subsequent singular components are partitioned among the specialized experts. A crucial element of this phase is the backbone weight adjustment scheme via residual subtraction, which preemptively deducts the expected initial contribution of the newly added experts from the pre-trained backbone. The second phase details the forward pass, which aggregates the adjusted frozen weights, the continuously active generalized expert, and the dynamically routed top-$k$ specialized experts.

\begin{algorithm}[h]
\caption{GSE: Generalized and Specialized Experts Initialization and Forward Pass}
\label{alg:gse}
\begin{algorithmic}[1]
\Require Input $x$, pre-trained weight $W_0$, rank $r_g$ for the generalized expert, rank $d$ for specialized experts, number of specialized experts $E$, generalized scaling factor $s_g$, and specialized scaling factors $\{s_s^i\}_{i=1}^{E}$
\State \textbf{Initialization Phase:}
\State Compute SVD: $W_0 = U \Sigma V^\top$
\State \Comment{1. Initialize generalized expert (largest singular values)}
\State $B_g \gets \sqrt{\frac{1}{s_g}} U_{1:r_g} \Sigma_{1:r_g}^{1/2}$
\State $A_g \gets \sqrt{\frac{1}{s_g}} \Sigma_{1:r_g}^{1/2} (V_{1:r_g})^\top$
\State \Comment{2. Initialize specialized experts (subsequent singular values)}
\For{$i = 1$ to $E$}
    \State Extract the $i$-th SVD segment $(U_i, \Sigma_i, V_i^\top)$ of rank $d$, starting after index $r_g$
    \State $B_s^i \gets \sqrt{\frac{1}{s_s^i}} U_i \Sigma_i^{1/2}$
    \State $A_s^i \gets \sqrt{\frac{1}{s_s^i}} \Sigma_i^{1/2} V_i^\top$
\EndFor
\State \Comment{3. Backbone weight adjustment via residual subtraction}
\State $W_{\mathrm{res}} \gets s_g B_g A_g + \frac{1}{E}\sum_{i=1}^{E} s_s^i B_s^i A_s^i$
\State $\tilde{W}_0 \gets W_0 - W_{\mathrm{res}}$
\State \Return $\tilde{W}_0, B_g, A_g, \{B_s^i, A_s^i\}_{i=1}^{E}$

\Statex
\State \textbf{Forward Pass Phase$(x)$:}
\State Compute gating scores $w^i(x)$ for specialized experts
\State Select top-$k$ specialized experts $\Omega_k(x)$
\State $y_{\mathrm{gen}} \gets s_g B_g A_g x$
\State $y_{\mathrm{spec}} \gets \sum_{i \in \Omega_k(x)} w^i(x)\, s_s^i\, B_s^i A_s^i x$
\State \Return $\tilde{W}_0 x + y_{\mathrm{gen}} + y_{\mathrm{spec}}$
\end{algorithmic}
\end{algorithm}

\section{Catastrophic Forgetting under VLM-to-VLA Full-Parameter Fine-Tuning}
\label{app:catastrophic_forgetting_vla}

A growing body of recent work has pointed out that directly adapting a pretrained VLM into a VLA by full-parameter fine-tuning can easily induce \emph{catastrophic forgetting}. Although the VLM-to-VLA paradigm is attractive because it transfers rich visual-semantic priors from internet-scale pretraining to embodied control, the optimization objective in robot learning is substantially narrower than that of the original VLM. As a result, the backbone can over-specialize to low-level action prediction and partially lose the open-world understanding ability that originally motivated using a VLM backbone in the first place.

Several recent studies provide direct evidence for this phenomenon. 
\citet{hancock2025actions} argue that catastrophic forgetting in VLM-to-VLA transfer is fundamentally related to the distribution mismatch between web-scale VLM pretraining data and robotics fine-tuning data. They show that learning to generate actions often diminishes the VLM's original reasoning and multimodal understanding ability, which in turn harms semantic generalization and instruction following. Their solution is primarily \emph{data-centric}: by re-expressing actions as language, they reduce representational mismatch and enable LoRA-based adaptation without heavily perturbing the pretrained backbone.

\citet{dey2024revla} provide complementary evidence from the vision backbone perspective. Focusing on OpenVLA, they show that fine-tuning the full robotic model can lead to severe forgetting in the visual encoders, especially DINO-v2. In particular, they demonstrate that the DINO-v2 encoder inside OpenVLA fails on depth regression after VLA training, in sharp contrast to the original pretrained backbone. Their findings suggest that end-to-end robotic fine-tuning may overwrite core spatial representations required for visual generalization. Their remedy is \emph{post hoc restoration}: they partially revert the visual encoders back to their pretrained states through backbone reversal and model merging.

\citet{yang2025instructvla} further observe that existing VLA models often suffer catastrophic forgetting of pretrained vision-language capabilities when adapted to manipulation. Their analysis shows that directly fine-tuning OpenVLA on instruction-rich embodied data yields limited gains on reasoning-intensive tasks because the original vision-language competence is not well preserved. To address this, they propose a \emph{training-centric} solution based on joint multimodal instruction tuning, where embodied data and general multimodal data are optimized together.

Most strikingly, \citet{yu2026twinbrainvla} provide explicit empirical analysis of the severity of forgetting during VLA training. They report that standard robot-only VLA fine-tuning can cause a near-complete collapse of general visual understanding; for example, the POPE score of Qwen3-VL drops from 88.87\% to 0.04\% after fine-tuning. They also show that simple co-training with general visual QA data only partially alleviates the issue. Their proposed solution is therefore \emph{architectural decoupling}: a frozen ``Left Brain'' preserves general semantic understanding, while a trainable ``Right Brain'' specializes in embodied control.

Our proposed VLA-GSE differs from the above lines of work in an important way. Rather than addressing forgetting primarily through data relabeling, co-training, post-training reversal, or duplicating the backbone into separate frozen/trainable branches, we approach the problem from the perspective of \emph{architectural improvement within the adaptation module itself}. Concretely, VLA-GSE is built on the hypothesis that the tension between general visual-semantic knowledge and task-specific embodied adaptation should be handled by a better parameterization of the trainable update, rather than by fully rewriting the backbone. The generalized expert is designed to preserve and reuse dominant pretrained directions that support broad transferable capability, while the specialized experts allocate dedicated capacity for task-dependent embodied adaptation. This design aims to improve specialization without forcing the backbone update to collapse into a single monolithic fine-tuning trajectory.

Therefore, compared with prior methods, VLA-GSE is not merely a forgetting \emph{mitigation strategy} added on top of standard adaptation; it is an architectural reformulation of the adaptation process itself. In this sense, our method is conceptually closer to improving the \emph{structure of transfer} than to modifying the training data distribution or repairing forgetting after it has already occurred. We believe this distinction is important: it suggests that preserving general VLM capability and acquiring embodied expertise need not be treated as a purely training-data trade-off, but can instead be addressed through a more suitable model architecture for VLM-to-VLA transfer.

\section{Detailed Explanation of the Load Balancing Auxiliary Loss}
\label{app:load_balancing}

Each GSE block adopts sparse routing over the {specialized experts}, which enables input-dependent adaptation while preserving computational efficiency. However, sparse top-$k$ routing is known to suffer from a degenerate optimization behavior in which only a small subset of experts is selected for most tokens, whereas the remaining experts receive little gradient signal. This \emph{routing collapse} effect reduces the effective capacity of the expert pool and weakens the intended specialization mechanism. To mitigate this issue, we introduce an auxiliary load-balancing objective following prior sparse MoE literature \citep{dai2024deepseekmoe,fedus2022switch}.

We emphasize that this regularizer is imposed only on the routed specialized experts. The \emph{Generalized Expert} is always activated and therefore does not participate in the routing competition.

For the $l$-th GSE block, consider a batch of $T$ tokens with token representations $\{x^{(t)}\}_{t=1}^{T}$. Let $\Omega_k(x^{(t)})$ denote the set of top-$k$ selected specialized experts for token $x^{(t)}$, and let $p^i(x^{(t)})$ denote the corresponding soft routing probability of expert $i$. We define the empirical selection frequency of expert $i$ as
\begin{equation}
f_i^{(l)} = \frac{1}{T} \sum_{t=1}^{T} \mathbb{I}\big(i \in \Omega_k(x^{(t)})\big),
\end{equation}
which measures how often expert $i$ is actually activated under hard routing. In parallel, we define the average routing mass of expert $i$ as
\begin{equation}
P_i^{(l)} = \frac{1}{T} \sum_{t=1}^{T} p^i(x^{(t)}),
\end{equation}
which reflects the router's continuous preference for that expert before the top-$k$ truncation.

The auxiliary balancing term for block $l$ is then given by
\begin{equation}
\mathcal{L}_{bal}^{(l)} = E \sum_{i=1}^{E} f_i^{(l)} P_i^{(l)},
\end{equation}
where $E$ is the number of specialized experts. This form couples the hard assignment statistics $f_i^{(l)}$ with the soft routing scores $P_i^{(l)}$, and therefore penalizes imbalance at both the discrete and probabilistic levels. In particular, if the router collapses onto a small number of experts, then both $f_i^{(l)}$ and $P_i^{(l)}$ concentrate on those experts, leading to a larger penalty. Conversely, when expert usage is more evenly distributed, both quantities become closer to uniform and the regularizer is correspondingly reduced.

The overall training objective is
\begin{equation}
\mathcal{L}
= \| \hat{a} - a \|_1
+ \alpha \sum_{l=1}^{L} \mathcal{L}_{bal}^{(l)},
\end{equation}
where $\| \hat{a} - a \|_1$ denotes the action regression loss and $\alpha$ controls the strength of the auxiliary regularization.

From an optimization perspective, this objective encourages agreement between the router's soft probability mass and its hard top-$k$ allocation statistics while disincentivizing concentration on a small subset of experts. Under balanced routing, one expects both $f_i^{(l)}$ and $P_i^{(l)}$ to approach a near-uniform distribution across experts, in which case the loss is minimized up to the top-$k$ sparsity constraint. Consequently, the regularizer improves training stability, prevents persistent expert under-utilization, and promotes more effective specialization across the expert pool. Empirically, this effect is important for maintaining the intended division of labor among specialized experts and for fully exploiting the parameter efficiency of the GSE design.

% \newpage
% \input{checklist.tex}

\end{document}